\documentclass{article}

\usepackage[final, nonatbib]{neurips_2021}

\usepackage[utf8]{inputenc} 
\usepackage[T1]{fontenc}    
\usepackage{hyperref}       
\def\equationautorefname~#1\null{Eq.~(#1)\null}
\def\figureautorefname~#1\null{Fig.~#1\null}
\def\tableautorefname~#1\null{Table~#1\null}
\def\sectionautorefname~#1\null{Section~#1\null}
\usepackage{url}            
\usepackage{booktabs}       
\usepackage{amsfonts}       
\usepackage{nicefrac}       
\usepackage{microtype}      
\usepackage{xcolor}         
\usepackage[pdftex]{graphicx}         

\definecolor{alessandro}{RGB}{255,0,0}
\definecolor{francesco}{RGB}{0,255,0}
\definecolor{matthieu}{RGB}{0,0,255}

\usepackage{bm}
\usepackage{amsmath}
\usepackage{amssymb}
\usepackage{amsthm}

\DeclareMathOperator*{\argmin}{argmin}  
\DeclareMathOperator{\eq}{\,{=}\,}

\newtheorem{theorem}{Theorem}[section]
\newtheorem{corollary}{Corollary}[theorem]
\newtheorem{lemma}[theorem]{Lemma}
\newtheorem{definition}{Definition}[section]

\title{Locality defeats the curse of dimensionality in convolutional teacher-student scenarios}

\author{
    Alessandro Favero \thanks{Equal contribution.} \\
    Institute of Physics\\
    \'Ecole Polytechnique F\'ed\'erale de Lausanne\\
    \texttt{alessandro.favero@epfl.ch} \\
    \And
    Francesco Cagnetta \footnotemark[1] \\
    Institute of Physics\\
    \'Ecole Polytechnique F\'ed\'erale de Lausanne\\
    \texttt{francesco.cagnetta@epfl.ch} \\
   \And
    Matthieu Wyart \\
    Institute of Physics\\
    \'Ecole Polytechnique F\'ed\'erale de Lausanne\\
    \texttt{matthieu.wyart@epfl.ch} \\
}

\begin{document}

\maketitle

\begin{abstract}

   Convolutional neural networks perform  a local and translationally-invariant treatment of the data: quantifying which of these two aspects is central to their success remains a challenge. We study this problem within a teacher-student framework for kernel regression, using `convolutional' kernels inspired by the neural tangent kernel of simple convolutional architectures of given filter size. Using heuristic methods from physics, we find in the ridgeless case that locality is key in determining the learning curve exponent $\beta$ (that relates the test error $\epsilon_t\sim P^{-\beta}$ to the size of the training set $P$), whereas translational invariance is not. In particular, if the filter size of the teacher $t$ is smaller than that of the student $s$, $\beta$ is a function of $s$ only and does not depend on the input dimension. We confirm our predictions on $\beta$ empirically.  We conclude by proving, under a natural universality assumption, that performing kernel regression with a ridge that decreases with the size of the training set leads to  similar learning curve exponents to those we obtain in the ridgeless case.
   
\end{abstract}

\addtocontents{toc}{\setcounter{tocdepth}{0}}

\section{Introduction}

Deep Convolutional Neural Networks (CNNs) are widely recognised as the engine of the latest successes of deep learning methods, yet such a success is surprising. Indeed, any supervised learning model suffers \emph{in principle} from the curse of dimensionality: under minimal assumptions on the function to be learnt, achieving a fixed target generalisation error $\epsilon$ requires a number of training samples $P$ which grows exponentially with the dimensionality $d$ of input data~\cite{luxburg2004distance}, i.e.  $\epsilon(P) \sim P^{-1/d}$. Nonetheless,  empirical evidence shows that the curse of dimensionality is beaten~\emph{in practice}~\cite{hestness2017deep,krizhevsky2012imagenet, spigler2019asymptotic}, with
\begin{equation}\label{eq:learning-scaling}
\epsilon(P) \sim P^{-\beta}, \quad \beta\,{\gg}\,1/d.
\end{equation}
CNNs, in particular, achieve excellent performances on high-dimensional tasks such as image classification on ImageNet with state-of-the-art architectures, for which $\beta\approx [0.3, 0.5]$ ~\cite{hestness2017deep}.
Natural data must then possess additional structures that make them learnable. A classical idea \cite{biederman1987recognition} ascribes the success of recognition systems to the compositionality of data, i.e. the fact that objects are made of features, themselves made of sub-features~\cite{poggio2017and, deza2020hierarchically,bietti2021approximation}. In this view, the locality of CNNs plays a key role for their performance, as supported by empirical observations ~\cite{Neyshabur2020towards}. Yet, there is no clear analytical understanding of the relationship between the compositionality of the data and learning curves. 

In order to study this relationship quantitively, we introduce a teacher-student framework for kernel regression, where the function to be learnt takes one of the following two forms:
\begin{equation}\label{eq:loc-f}
    f^{LC}(\bm{x}) = \sum_{i\in \mathcal{P}} g_i(\bm{x}_i), \ \ \ f^{CN}(\bm{x}) = \sum_{i\in \mathcal{P}} g(\bm{x}_i).
\end{equation}
Here, $\bm{x}$ is a $d$-dimensional input and $\bm{x}_i$ denotes the $i$-th $t$-dimensional patch of $\bm{x}$, $\bm{x}_i\,{=}\,(x_i, \dots, x_{i+t-1})$. $i$ ranges in a subset $\mathcal{P}$ of $\left\lbrace 1,\dots, d\right\rbrace$. The $g_i$'s and $g$ are random functions of $t$ variables whose smoothness is controlled by some exponent $\alpha_t$. Such functions model the local nature of certain datasets and can be generated, for example, by randomly-initialised one-hidden-layer neural networks: $f^{LC}$ corresponds to a \emph{locally connected} network (LCN)~\cite{fukushima1975cognitron, lecun1989generalization}, in which the input is split into lower-dimensional patches before being processed, whereas a network enforcing invariance with respect to shifts of the input patches via weight sharing can be described by $f^{CN}$. 
In such cases $t$ would be the filter size of the network. Our goal is to compute the asymptotic decay of the error of a student kernel performing regression on such data, and to relate the corresponding exponent $\beta$ to the locality of the target function.
The student kernel corresponds to a prior on the true function of the form described by \autoref{eq:loc-f}, except that the filter size $s$ and its prior $\alpha_s$ on the smoothness of the $g$ functions can differ from those of the target function. Such students include overparametrised one-hidden-layer neural networks operating in the \emph{lazy training regime}~\cite{jacot2018neural, Du2019, lee2019wide, arora2019exact, chizat2019lazy}.

\subsection{Our contributions}

We consider a teacher-student framework for kernel regression, where the target function has one of the forms in~\autoref{eq:loc-f}, where the $g_i$'s and $g$ are Gaussian random fields of given covariance. Target functions are characterised by the dimensionality $t$ of the $g$ functions---the {\it filter size}---and a smoothness exponent $\alpha_t$, such that $\alpha_t > 2n$ implies that typical target functions are at least $n$ times differentiable.
Kernel regression is performed by {\it local} or {\it convolutional} student kernels, having filter size $s$ and a prior on the target smoothness characterised by another exponent $\alpha_s\,{>}\,0$. Our main contributions follow:

\begin{itemize}
 \item[$\circ$] We use recent results based on the replica method of statistical physics  on the generalisation error of kernel methods \cite{bordelon2020spectrum, canatar2021spectral,loureiro2021capturing} to estimate the exponent $\beta$.  We find that $\beta=\alpha_t/s$ if $t\leq s$ and $\alpha_t\leq 2(\alpha_s+s)$. This approach is non-rigorous, but it can be proven if data are sampled on a lattice \cite{spigler2019asymptotic} and corresponds to a provable lower bound on the error when teacher and student are equal \cite{micchelli1979design}.
 \item[$\circ$] In particular, we find the same exponent for students with a prior on the shift invariance of the target function and students without this prior, implying that the curse of dimensionality is beaten due to locality and not shift invariance.
 \item[$\circ$] We confirm systematically our predictions by performing kernel ridgeless regression numerically for various $t$, $s$ and embedding dimension $d$.
 \item[$\circ$] We use the recent framework of \cite{jacot2020kernel} and a natural Gaussian universality assumption to prove a rigorous estimate of $\beta$ in the  case where the ridge  decreases with the size of the training set. The estimate of $\beta$ depends again on $s$ and not on $d$, demonstrating that the curse of dimensionality can indeed be beaten by using local filters on such compositional data.
\end{itemize}

\subsection{Related work}

Several recent works study the role of the compositional structure of data~\cite{poggio2017and, kondor2018generalization, zhou2018building}. When such structure is hierarchical, deep convolutional networks can be much more expressive than shallow ones \cite{poggio2017and, Poggio2020theo, deza2020hierarchically}. Concerning training, ~\cite{malach2021computational} shows that both convolutional and locally-connected networks can achieve a target generalisation error in polynomial time, whereas fully-connected networks cannot, for a class of functions which depend only on $s$ consecutive bits of the $d$-dimensional input, with $s\,{=}\mathcal{O}(\log{d})$. In~\cite{bietti2021approximation} the effects of the architecture's locality are studied from a kernel perspective, using a class of deep convolutional kernels introduced in~\cite{mairal2016end, bietti2019inductive} and characterising their Reproducing Kernel Hilbert Space (RKHS). In general, belonging to the RKHS ensures favourable bounds on performance and, for isotropic kernels, is a constraint on the function smoothness that becomes stringent in large $d$. For local functions, the corresponding constraint on smoothness is governed by the filter size $s$ and not $d$ ~\cite{bietti2021approximation}. Lastly, a recent work shows that weight sharing, in the absence of locality, leads to a mild improvement of the generalisation error of shift-invariant kernels~\cite{mei2021learning}.

By contrast, our work focuses on computing non-trivial training curve exponents in a setup where the locality and  shift-invariance priors of the kernel can differ from those of the class of functions being learnt. In our setup, the latter are in general not in the RKHS of the kernel\footnote{A Gaussian field of covariance $K$ is never in the RKHS of the kernel $K$, see e.g. \cite{kanagawa2018gaussian}.}.
Technically, our result that the size of the student filter $s$ controls the learning curve (and not that of the teacher $t$) relates to the fact that kernels are not able to detect data anisotropy (the fact that the function depends only on a subset of the coordinates) in worst-case settings  \cite{bach2017breaking} nor in the typical case for Gaussian fields \cite{paccolat2020isotropic}.

\section{Setup}\label{sec:setup}

\paragraph{Kernel ridge regression}
Kernel ridge regression is a method to learn a target function $f^*:\mathbb{R}^d\to \mathbb{R}$ from $P$ observations $\{(\bm{x}^\mu, f^*(\bm{x}^\mu))\}_{\mu=1}^P$, where the inputs $\bm{x}^\mu$ are i.i.d. random variables distributed according to a certain measure $p\left( d^d x\right)$ on $\mathbb{R}^d$. Let $K$ be a positive-definite kernel and $\mathcal{H}$ the corresponding Reproducing Kernel Hilbert Space (RKHS). The kernel ridge regression estimator $f$ of the target function $f^*$ is defined as
\begin{equation}\label{eq:argmin}
 f =\argmin_{f\in \mathcal{H}} \left\lbrace \frac{1}{P}\displaystyle \sum_{\mu=1}^P \left(f(\bm{x}^\mu) - f^*(\bm{x}^\mu)\right)^2  + \lambda \, \|f\|^2_{\mathcal{H}} \right  \rbrace,
\end{equation}
where $\|\cdot\|_{\mathcal{H}}$ denotes the RKHS norm and $\lambda$ is the ridge parameter. The limit $\lambda\to0^+$ is known as the ridgeless case and corresponds to the solution with minimum RKHS norm that interpolates the $P$ observations.
\autoref{eq:argmin} is a convex optimisation problem, having the unique solution
\begin{equation}\label{eq:krrpredictor}
 f(\bm{x}) = \frac{1}{P}\sum_{\mu,\nu=1}^P K(\bm{x},\bm{x}^\mu) \left(\left( \frac{1}{P} \mathbb{K}_P +\lambda \mathbb{I}_P\right)^{-1}\right)_{\mu,\nu} f^*({\bm{x}^\nu}), 
\end{equation}
where $\mathbb{K}_P$ is the \emph{Gram matrix} defined as $(\mathbb{K}_P)_{\mu\nu} = K(\bm{x}^\mu, \bm{x}^\nu)$, and $\mathbb{I}_P$ denotes the $P$-dimensional identity matrix. Our goal is to compute the generalisation error, which we define as the expectation of the mean squared error over the data distribution $p\left( d^d x\right)$, averaged over an ensemble of target functions $f^*$, i.e
\begin{equation}\label{eq:test-def}
\epsilon(P) = \mathbb{E}_{\bm{x},f^*}\left[ \left( f(\bm{x})-f^*(\bm{x})\right)^2 \right].
\end{equation}
The error $\epsilon$ depends on the number of samples $P$ through the predictor of~\autoref{eq:krrpredictor} and we refer to the graph of $\epsilon(P)$ as {\it learning curve}.

\paragraph{Statistical mechanics of generalisation in kernel regression}
The theoretical understanding of generalisation is still an open problem. A few recent works~\cite{bordelon2020spectrum, jacot2020kernel, canatar2021spectral} relate the generalisation error $\epsilon$ to the decomposition of the target function in the eigenbasis of the kernel. A positive-definite kernel $K$ can indeed be written, by Mercer's theorem, in terms of its eigenvalues $\{\lambda_\rho\}$ and eigenfunctions $\{\phi_\rho\}$:
\begin{equation}\label{eq:mercer}
 K(\bm{x},\bm{y}) = \sum_{\rho=1}^\infty \lambda_\rho \phi_\rho(\bm{x}) \overline{\phi_\rho(\bm{y})}, \quad \int p\left( d^d y\right) K(\bm{x},\bm{y}) \phi_\rho(\bm{y}) = \lambda_\rho \phi_\rho(\bm{x}). 
\end{equation}
In~\cite{bordelon2020spectrum, jacot2020kernel, canatar2021spectral} it is shown that, when the target function can be written in terms of the kernel eigenbasis,
\begin{equation} 
f^*(\bm{x}) = \sum_\rho c_\rho \phi_\rho(\bm{x}),
\end{equation}
the error $\epsilon$ can also be cast as a sum of modal contributions, $\epsilon = \sum_\rho \epsilon_\rho$. The details of the general formulation are summarised in~\autoref{app:spectral}. Here we present an intuitive limiting case, obtained in the ridgeless limit $\lambda\to 0^+$, when $\lambda_\rho \sim \rho^{-a}$ for large $\rho$, and $\mathbb{E}[|c_\rho|^2] \sim \rho^{-b}$ with $2a\,{>}\,b-1$, that is
\begin{equation}\label{eq:error-scaling}
\epsilon(P) \sim \sum_{\rho > P} \mathbb{E}[|c_\rho|^2] \equiv \mathcal{B}(P),
\end{equation}
with $\sim$ denoting asymptotic equivalence for large $P$. \autoref{eq:error-scaling} indicates that, given $P$ examples, the generalisation error can be estimated as the tail sum of the power in the target function past the first $P$ modes of the kernel, which we denote as $\mathcal{B}(P)$. Although the general modal decomposition cannot be proven rigorously in the ridgeless limit~\cite{jacot2020kernel, loureiro2021capturing}, additional results are available when the target functions are Gaussian random fields with covariance specified by a teacher kernel:
\begin{itemize}
    \item[$\circ$] \autoref{eq:error-scaling} can be proven rigorously~\cite{spigler2019asymptotic} if teacher and student are isotropic kernels and the input points $\bm{x}^\mu$ are sampled on the lattice $\mathbb{Z}^d$, i.e. all the elements of each input sequence are integer multiples of an arbitrary unit;
    \item[$\circ$] If teacher and student coincide then $\mathbb{E}[|c_\rho|^2]$ equals the $\rho$-th eigenvalue  $\lambda_\rho$ and (see e.g.~\cite{micchelli1979design}) $\epsilon(P) \geq \mathcal{B}(P)$, i.e. the estimate of~\autoref{eq:error-scaling} is a lower bound.
\end{itemize}

\section{Kernels for local and convolutional teacher-student scenarios}\label{sec:convolutional-mercer}

In this section we introduce convolutional and local kernels that will be used as teachers, i.e. to generate different ensembles of target functions $f^*$ with controlled smoothness and degree of locality, and as student kernels. We motivate our choice by considering one-hidden-layer neural networks with simple local and convolutional architectures. Because of the relationship between our kernels and the Neural Tangent Kernel~\cite{jacot2018neural} of the aforementioned architectures, our framework encompasses regression with simple overparametrised networks trained in the lazy regime~\cite{chizat2019lazy}.
For the sake of clarity we limit the discussion to inputs which are sequences in $\mathbb{R}^d$, i.e. $\bm{x} \,{=}\,(x_1,\dots,x_d)$. Extension to higher-order tensorial inputs such as images $\bm{X}\in\mathbb{R}^{d\times d}$ is straightforward. To avoid dealing with the boundaries of the sequence we identify $x_{i+d}$ with $x_i$ for all $i\,{=}\,1,\dots,d$.

\begin{definition}[one-hidden-layer CNN]\label{eq:cnn} 
A one-hidden-layer convolutional network with $H$ hidden neurons and average pooling is defined as follows,
\begin{equation}\label{eq:cnn-out}
    f^{CNN}(\bm{x}) = \frac{1}{\sqrt{H}} \sum_{h=1}^H a_h \frac{1}{|\mathcal{P}|}\sum_{i\in\mathcal{P}} \sigma(\bm{w}_{h}\cdot \bm{x}_i),
\end{equation}
where $\bm{x}\in\mathbb{R}^d$ is the input, $H$ is the width, $\sigma$ a nonlinear activation function, $\mathcal{P}\subseteq\left\lbrace{1,\dots,d}\right\rbrace$ is a set of patch indices and $|\mathcal{P}|$ its cardinality. For all $i\in\mathcal{P}$, $\bm{x}_i$ is an $s$-dimensional patch of $\bm{x}$. For all $h\,{=}\,1,\dots,H$, $\bm{w}_{h}\in\mathbb{R}^s$ is a filter with filter size $s$, $a_{h}\in\mathbb{R}$ is a scalar weight. The dot $\cdot$ denotes the standard Euclidean scalar product.
\end{definition}
In the network defined above, a $d$-dimensional input sequence $\bm{x}$ is first mapped to $s$-dimensional \emph{patches} $\bm{x}_i$, which are ordered subsequences of the input. Comparing each patch to a filter $\bm{w}_{h}$ and applying the activation function $\sigma$ leads to a $|\mathcal{P}|$-dimensional hidden representation which is equivariant for shifts of the input. The summation over the patch index $i$ promotes this equivariance to full invariance, leading to a model which is both local and shift-invariant as $f^{CN}$ in~\autoref{eq:loc-f}. A model which is only local, as $f^{LC}$ in~\autoref{eq:loc-f}, can be obtained by lifting the constraint of weight-sharing, which forces, for each $h\,{=}\,1,\dots,H$, the same filter $\bm{w}_{h}$ to apply to all patches $\bm{x}_i$.
\begin{definition}[one-hidden-layer LCN]\label{eq:lcn} 
In the notation of Definition~\ref{eq:cnn}, a one-hidden-layer locally-connected network with $H$ hidden neurons is defined as follows,
\begin{equation}\label{eq:lcn-out}
    f^{LCN}(\bm{x}) = \frac{1}{\sqrt{H}} \sum_{h=1}^H \frac{1}{\sqrt{|\mathcal{P}|}}\sum_{i\in\mathcal{P}} a_{h,i} \sigma(\bm{w}_{h,i}\cdot \bm{x}_i),
\end{equation}
For all $i\in\mathcal{P}$ and $h\,{=}\,1,\dots,H$: $\bm{x}_i$ is an $s$-dimensional patch of $\bm{x}$, $\bm{w}_{h,i}\in\mathbb{R}^s$ is a filter with filter size $s$, $a_{h,i}\in\mathbb{R}$ is a scalar weight. \end{definition}
Notice that the definition above reduces to that of a fully-connected network when the filter size is set to the input dimension, $s\,{=}\,d$, and $\mathcal{P}=\left\lbrace1\right\rbrace$. With the target functions taking one of the two forms in~\autoref{eq:loc-f}, our framework contains the case where the observations are generated by neural networks such as~(\ref{eq:cnn}) and~(\ref{eq:lcn}). Let us now introduce the neural tangent kernels of such architectures.

\begin{definition}[Neural Tangent Kernel]\label{eq:finite_ntk} Given a neural network function $f(\bm{x};\bm{\theta})$, where $\bm{\theta}\,{=}\,(\theta_1,\dots,\theta_N)$ denotes the complete set of parameters and $N$ the total number of parameters, the Neural Tangent Kernel (NTK) is defined as~\cite{jacot2018neural}
\begin{equation}
    \Theta_{N}(\bm{x},\bm{y};\bm{\theta}) = \sum_{n=1}^N \partial_{\theta_n} f(\bm{x},\bm{\theta})\partial_{\theta_n} f(\bm{y},\bm{\theta}),
\end{equation}
where $\partial_{\theta_n}$ denotes partial derivation w.r.t. the $n$-th parameter $\theta_n$.
\end{definition}
For one-hidden-layer networks with random, $\mathcal{O}(1)$-variance Gaussian initialisation of all the weights, and normalisation by $\sqrt{H}$ as in~(\ref{eq:cnn}) and~(\ref{eq:lcn}), the NTK converges to a deterministic limit $\Theta(\bm{x},\bm{y})$ as $N\propto H \to\infty$ ~\cite{jacot2018neural}. Furthermore, training $f(\bm{x},\bm{\theta})-f(\bm{x},\bm{\theta}_0)$, with $\bm{\theta}_0$ denoting the network parameters at initialisation, under gradient descent on the mean squared error is equivalent to performing ridgeless regression with kernel $\Theta(\bm{x},\bm{y})$~\cite{jacot2018neural}. The following lemmas relate the NTK of convolutional and local architectures acting on $d$-dimensional inputs to that of a fully-connected architecture acting on $s$-dimensional inputs. Both lemmas are proved in~\autoref{app:ntk}.
\begin{lemma}\label{lemma:cntk}
Call $\Theta^{FC}$ the NTK of a fully-connected network function acting on $s$-dimensional inputs and $\Theta^{CN}$ the NTK of a convolutional network function~(\ref{eq:cnn}) with filter size $s$ acting on $d$-dimensional inputs. Then
\begin{equation}\label{eq:conv-ntk}
    \Theta^{CN}(\bm{x},\bm{y}) = \frac{1}{|\mathcal{P}|^2}\sum_{i,j\in\mathcal{P}} \Theta^{FC}(\bm{x}_i,\bm{y}_j)
\end{equation}
\end{lemma}
As the functions in~\autoref{eq:loc-f}, $\Theta^{CN}$ is written as a combination of lower-dimensional constituent kernels $\Theta^{FC}$ acting on patches, and the dimensionality of the constituent kernel coincides with the filter size of the corresponding network. This observation extends to local kernels, via
\begin{lemma}\label{lemma:lntk}
Call $\Theta^{LC}$ the NTK of a locally-connected network function~(\autoref{eq:lcn}) with filter size $s$ acting on $d$-dimensional inputs. Then
\begin{equation}\label{eq:loc-ntk}
    \Theta^{LC}(\bm{x},\bm{y}) = \frac{1}{|\mathcal{P}|}\sum_{i\in\mathcal{P}} \Theta^{FC}(\bm{x}_i,\bm{y}_i)
\end{equation}
\end{lemma}

Following the general structure of~\autoref{eq:conv-ntk} and~\autoref{eq:loc-ntk}, we introduce convolutional ($K^{CN}$) and local ($K^{LC}$) student and teacher kernels, defined as sums of lower-dimensional constituent kernels $C$,
\begin{subequations}\label{eq:convloc-ker}
\begin{align}
    \label{eq:conv-ker} K^{CN}(\bm{x},\bm{y}) &= |\mathcal{P}|^{-2} \displaystyle\sum_{i,j\in\mathcal{P}} C(\bm{x}_i, \bm{y}_j),\\
    \label{eq:loc-ker} K^{LC}(\bm{x},\bm{y}) &= |\mathcal{P}|^{-1} \displaystyle\sum_{i\in\mathcal{P}} C(\bm{x}_i, \bm{y}_i).
\end{align}
\end{subequations}
The kernels in~\autoref{eq:convloc-ker} are characterised by the dimensionality of the constituent kernel $C$, or \emph{filter size} $s$ (for the student, or $t$ for the teacher) and the nonanalytic behaviour of $C$ when the two arguments approach, i.e. $C(\bm{x}_i,\bm{y}_j)\sim \|\bm{x}_i-\bm{y}_j\|^{\alpha_{s}}$ (for the student, or $\|\bm{x}_i-\bm{y}_j\|^{\alpha_{t}}$ for the teacher) plus analytic contributions, with $\alpha_{s/t}\neq 2m$ for $m\in\mathbb{N}$. Using the kernels in~\autoref{eq:convloc-ker} as covariances allows us to generate random target functions with the desired degree of locality $t$ (as in~\autoref{eq:loc-f}), which can also be invariant for shifts of the patches. Having a student kernel as in~\autoref{eq:convloc-ker} results in an estimator $f$ also having the form displayed in~\autoref{eq:loc-f}, with a different filter size with respect to the target function. The $\alpha$'s control the smoothness of these functions as, if $\alpha\,{>}\,2 n \in \mathbb{N}$, then the functions are at least $n$ times differentiable in the mean-square sense.

A notable example of such constituent kernels is the NTK of ReLU networks $\Theta^{FC}$, which presents a cusp at the origin corresponding to $\alpha_s\,{=}\,1$ \cite{geifman2020similarity}. In addition, in the $H\to\infty$ limit, a network initialised with random weights converges to a Gaussian process~\cite{Neal1996, williams1997computing, lee2017deep}. For networks with ReLU activations, the covariance kernel of such process has nonanalytic behaviour with $\alpha_t\,{=}\,3$~\cite{cho2009kernel}.

\subsection{Mercer's decomposition of local and convolutional kernels}

We now turn to describing how the eigendecomposition of the constituent kernel $C$ induces an eigendecomposition of convolutional and local kernels. We work under the following assumptions,
\begin{itemize}
    \item[$i)$] The constituent kernel $C(\bm{x}, \bm{y})$ on $\mathbb{R}^s\times \mathbb{R}^s$ admits the following Mercer's decomposition,
    \begin{equation}\label{eq:mercer-c}
        C(\bm{x}, \bm{y}) = \sum_{\rho=1}^\infty \lambda_\rho \phi_\rho(\bm{x})\phi_\rho(\bm{y}),
    \end{equation}
    with (ordered) eigenvalues $\lambda_{\rho}$ and eigenfunctions ${\phi_\rho}$ such that, with $p^{(s)}(d^s x)$ denoting the $s$-dimensional patch measure, $ \phi_1(\bm{x}) = 1 \; \forall \bm{x}$ and $ \int p^{(s)}(d^s x) \phi_\rho(\bm{x})\,{=}\,0$ for all $\rho\,{>}1$;
    \item[$ii)$] Convolutional and local kernels from \autoref{eq:convloc-ker} have \emph{nonoverlapping} patches, i.e. $d$ is an integer multiple of $s$ and $\mathcal{P}\,{=}\,\left\lbrace 1 + n\times s \, | \, n=1,\dots,d/s \right\rbrace$ with $|\mathcal{P}|\,{=}d/s$;
    \item[$iii)$] The $s$-dimensional marginals on patches of the $d$-dimensional input measure $p^{(d)}(d^d x)$ are all identical and equal to $p^{(s)}(d^s x)$.
\end{itemize}
We stress here that the request of nonoverlapping patches in assumption $ii)$ can be relaxed at the price of further assumptions, i.e. $C(\bm{x},\bm{y})\,{=}\,\mathcal{C}(\bm{x}-\bm{y})$ and data distributed uniformly on the torus, so that $C$ is diagonalised in Fourier space. The resulting eigendecompositions are qualitatively similar to those described in this section (details in~\autoref{app:mercer-overlap}). Let us also remark that assumptions $i)$ and $iii)$---together with all the assumptions on the data distribution that might follow---are technical in nature and required only to carry out the Mercer's decomposition analytically. We believe that the main results of this paper hold under much more general conditions, namely the support of the distribution being truly $d$-dimensional---such that the distance between neighbouring points in a collection of $P$ data points scales as $P^{-1/d}$---and the distribution itself decaying rapidly away from the mean or having compact support. Our experiments, discussed in~\autoref{sec:empirical}, support this hypothesis.

\begin{lemma}[Spectra of convolutional kernels]\label{lemma:conv-spectra}
Let $K^{CN}$ be a convolutional kernel defined as in~\autoref{eq:conv-ker}, with a constituent kernel $C$ satisfying assumptions $i)$, $ii)$ and $iii)$ above. Then $K^{CN}$ admits the following Mercer's decomposition,
\begin{equation}\label{eq:conv-decomp}
K^{CN}(\bm{x},\bm{y}) =\sum_{\rho=1}^{\infty} \Lambda_\rho \Phi_\rho(\bm{x}) \overline{\Phi_\rho(\bm{y})},
\end{equation}
with eigenvalues and eigenfunctions
\begin{equation}\label{eq:conv-spectrum}
  \Lambda_1 = \lambda_1,\,  \Phi_1(\bm{x}) = 1;\,\, \Lambda_\rho = \frac{s}{d}\lambda_\rho,\,  \Phi_\rho(\bm{x}) = \sqrt{\frac{s}{d}}\sum_{i\in\mathcal{P}}\phi_\rho(\bm{x}_i)  \text{ for }\rho > 1.
\end{equation}
\end{lemma}

\begin{lemma}[Spectra of local kernels]\label{lemma:loc-spectra}
Let $K^{LC}$ be a local kernel defined as in~\autoref{eq:loc-ker}, with a constituent kernel $C$ satisfying assumptions $i)$, $ii)$ and $iii)$ above. Then $K^{LC}$ admits the following Mercer's decomposition,
\begin{equation}\label{eq:loc-decomp}
K^{LC}(\bm{x},\bm{y}) = \Lambda_1 \Phi_1(\bm{x}) +  \sum_{\rho>1}^{\infty}\sum_{i\in\mathcal{P}} \Lambda_{\rho,i} \Phi_{\rho,i}(\bm{x}) \overline{\Phi_{\rho,i}(\bm{y})},
\end{equation}
with eigenvalues and eigenfunctions ($\forall i\in \mathcal{P}$)
\begin{equation}\label{eq:loc-spectrum}
  \Lambda_{1} = \lambda_1,\,  \Phi_{1}(\bm{x}) = 1;\,\, \Lambda_{\rho,i} = \frac{s}{d}\lambda_\rho,\,  \Phi_{\rho,i}(\bm{x}) = \phi_\rho(\bm{x}_i)  \text{ for }\rho > 1.
\end{equation}
\end{lemma}

Under assumptions $i)$, $ii)$ and $iii)$ above, lemmas~\ref{lemma:conv-spectra} and~\ref{lemma:loc-spectra} follow from the definitions of convolutional and local kernels and the eigendecompositions of the constituents (see~\autoref{app:mercer-overlap} for a proof of the lemmas and generalisation to kernels with overlapping patches). In the next section, we explore the consequences of these results for the asymptotics of learning curves.

\section{Asymptotic learning curves for ridgeless regression}\label{sec:learning-curves}

In what follows, we consider explicitly translationally-invariant constituent kernels $C(\bm{x}_i,\bm{x}_j)\,{=}\, \mathcal{C}(\bm{x}_i - \bm{x}_j)$ and a $d$-dimensional data distribution $p(d^dx)$ which is uniform on the torus, so that all lower-dimensional marginals are also uniform on lower-dimensional tori. Under these conditions, all results of~\autoref{sec:convolutional-mercer} can be extended to kernels with overlapping patches ($\mathcal{P}\,{=}\,\left\lbrace 1,\dots, d \right\rbrace$), so that the main results of this paper apply to nonoverlapping as well as overlapping-patches kernels. Furthermore, Mercer's decomposition~\autoref{eq:mercer-c} can be written in Fourier space~\cite{scholkopf2001learning}, with $s$-dimensional plane waves $\phi^{(s)}_{\bm{k}}(\bm{x})\,{=}\,e^{i\bm{k}\cdot\bm{x}}$ as eigenfunctions and the eigenvalues coinciding with the Fourier transform of $\mathcal{C}$. Furthermore, for kernels with filter size $s$  (or $t$) and positive smoothness exponent $\alpha_s$ (or $\alpha_t$), the eigenvalues decay with a power $-(s\,{+}\,\alpha_s)$ (or $-(t\,{+}\,\alpha_t)$) of the modulus of the wavevector $k\,{=}\,\sqrt{\bm{k}\cdot\bm{k}}$~\cite{widom1964asymptotic}. In this setting, we obtain our main result: 
\begin{theorem}\label{th:scaling}
Let $K_T$ be a $d$-dimensional convolutional kernel with a translationally-invariant $t$-dimensional constituent and leading nonanalyticity at the origin controlled by the exponent $\alpha_t\,{>}\,0$. Let $K_S$ be a $d$-dimensional convolutional or local student kernel with a translationally-invariant $s$-dimensional constituent, and with a nonanalyticity at the origin controlled by the exponent $\alpha_s\,{>}\,0$.
Assume, in addition, that  if the kernels have overlapping patches then $s\geq t$, whereas if the kernels have nonoverlapping patches $s$ is an integer multiple of $t$; and that data are uniformly distributed on a $d$-dimensional torus. Then, the following asymptotic equivalence holds in the limit $P\to\infty$,
$$\mathcal{B}(P) \sim P^{-\beta}, \quad \beta = \alpha_t / s.$$
\end{theorem}
\autoref{th:scaling}, together with \autoref{eq:error-scaling} and the additional assumption $\alpha_t\,{\leq}\, 2(\alpha_s + s)$, yields the following expression for the learning curves asymptotics, 
\begin{equation}\label{eq:prediction}
    \epsilon(P) \sim P^{-\beta}, \quad \beta = \alpha_t / s.
\end{equation}
As $\beta$ is independent of the embedding dimension $d$, we conclude that the curse of dimensionality is beaten when a convolutional target is learnt with a convolutional or local kernel. In fact, ~\autoref{eq:prediction} indicates that there is no asymptotic advantage in using a convolutional rather than local student when learning a convolutional task, confirming the picture that locality, not weight sharing, is the main source of the convolutional architecture's performances~\cite{poggio2017and}. In \autoref{app:thm1} we show that the generalization error of a local student learning convolutional teacher decays as
\begin{equation}\label{eq:prediction-local-stud}
    \epsilon(P) \sim \left(\frac{P}{|\mathcal{P}|}\right)^{-\beta}, \quad \beta = \alpha_t / s.
\end{equation}
\autoref{eq:prediction-local-stud} implies that including weight sharing only amounts to a rescaling of $P$ by a factor $|\mathcal{P}|$---the size of the translation group over patches---recovering the result obtained in \cite{mei2021learning}. 
Intuitively, a local student will need $|\mathcal{P}|$ times more points than a convolutional student to learn the target with comparable accuracy, since it has to learn the same local function in all the possible $|\mathcal{P}|$ locations.
The predictions in ~\autoref{eq:prediction} and \autoref{eq:prediction-local-stud} are confirmed empirically, as discussed in~\autoref{sec:empirical} and \autoref{app:numerics}. Let us mention in particular that, although our predictions are valid only asymptotically, they hold already in the range $P\sim 10^2-10^3$, consistently with the number of training points typically used in applications.

\autoref{th:scaling} is proven in~\autoref{app:thm1} and extended to the case of a local teacher and local student in~\autoref{app:local}. Here we sketch the proof for the nonoverlapping case, which begins with the calculation of the variance of the coefficients of the target function in the student basis. By indexing the coefficients with the $s$-dimensional wavevectors $\bm{k}$,
\begin{equation}\label{eq:coeff}\begin{aligned}
  \mathbb{E}[|c_{\bm{k}}|^2] &= \int_{[0,1]^d} d^d x \Phi_{\bm{k}}(\bm{x})\int_{[0,1]^d} d^d y \overline{\Phi_{\bm{k}}}(\bm{y}) \mathbb{E}[f^*(\bm{x})f^*(\bm{y})]\\ & =\int_{[0,1]^d} d^d x \Phi_{\bm{k}}(\bm{x})\int_{[0,1]^d} d^d y \overline{\Phi_{\bm{k}}}(\bm{y}) K_T(\bm{x}, \bm{y}). 
\end{aligned}\end{equation}
If the size of teacher and student coincide, $s\,{=}\,t$, teacher and student have the same eigenfunctions. Thus, using the eigenvalue equation \autoref{eq:mercer} of the teacher yields $\mathbb{E}[|c_{\bm{k}}|^2] \sim k^{-(\alpha_t + t)}\,{=}\,k^{-(\alpha_t + s)}$. After ranking eigenvalues by $k$, with multiplicity $k^{s-1}$ from all the wavevectors having the same modulus $k$, one has
\begin{equation}\label{eq:t-equal-s}
    \mathcal{B}(P) = \sum_{\left\lbrace \bm{k}| k > P^{1/s}\right\rbrace} k^{-(\alpha_t + s)} \sim \int_{P^{1/s}}^\infty dk k^{s-1} k^{-(\alpha_t + s)} \sim P^{-\frac{\alpha_t}{s}}.
\end{equation}
When the filter size of the teacher $t$ is lowered, some of the coefficients $\mathbb{E}[|c_{\bm{k}}|^2]$ vanish. As the target function becomes a composition of $t$-dimensional constituents, the only non-zero coefficients are found for $\bm{k}$'s which lie in some $t$-dimensional subspaces of the $s$-dimensional Fourier space. These subspaces correspond to the $\bm{k}$ having at most a patch of $t$ consecutive non-vanishing components. In other words, $\mathbb{E}[|c_{\bm{k}}|^2]$ is finite only if $\bm{k}$ is effectively $t$-dimensional and the integral on the right-hand side of~\autoref{eq:t-equal-s} becomes $t$-dimensional, thus
\begin{equation}\label{eq:t-smaller-s}
    \mathcal{B}(P) \sim \int_{P^{1/s}}^\infty dk k^{t-1} k^{-(\alpha_t + t)} \sim P^{-\frac{\alpha_t}{s}}.
\end{equation}

If the teacher patches are not contained in the student ones, the target cannot be represented with a combination of student eigenfunctions, hence the error asymptotes to a finite value when $P\to\infty$.

\section{Empirical learning curves for ridgeless regression}\label{sec:empirical}

This section investigates numerically the asymptotic behaviour of the learning curves for our teacher-student framework. We consider different combinations of convolutional and local teachers and students with overlapping patches and Laplacian constituent kernels, i.e. $\mathcal{C}(\bm{x}_i-\bm{x}_j) \, {=} \, e^{-\|\bm{x}_i-\bm{x}_j\|}$. In order to test the robustness of our results to the data distribution, data are uniformly generated in the hypercube $[0,1]^d$ (results in \autoref{fig:figure}) or on a $d$-hypersphere (results in~\autoref{app:numerics}). \autoref{fig:figure} shows learning curves for both convolutional (left panels) and local (right panels) students learning a convolutional target function. The results in the case of a local teacher are presented in \autoref{app:numerics}, and display no qualitative differences.

In the following, we always refer to \autoref{fig:figure}. Panels A and B show that, with $\alpha_t\,{=}\,\alpha_s\,{=}\,1$, our prediction $\beta\,{=}\,1/s$ holds independently of the embedding dimension $d$. Furthermore, notice that fixing the dimension $d$ and the teacher filter size $t$, the generalisation errors of a convolutional and a local student with the same filter size differ only by a multiplicative constant independent of $P$. Indeed, the shift-invariant nature of the convolutional student only results in a pre-asymptotic correction to our estimate of the generalisation error $\mathcal{B}(P)$. In \autoref{app:numerics}, we check that this multiplicative constant corresponds to rescaling $P$ by the number of patches, as predicted in \autoref{sec:learning-curves}. Panels C and D show learning curves for several values of $s$ and fixed $t$. The curse of dimensionality is recovered when the size of the student filters coincides with the input dimension, both for local and convolutional students. Finally, panels E and F show learning curves for fixed $t$ and $s$ being smaller than, equal to or larger than $t$. We stress that, when $s\,{<}\,t$ the student kernel cannot reproduce the target function, hence the error does not decrease by increasing $P$. Further details on the experiments are provided in~\autoref{app:numerics}, together with learning curves for data distributed uniformly on the unit sphere $\mathbb{S}^{d-1}$ and for regression with the actual analytical and empirical NTKs of one-hidden-layer convolutional networks. It is worthwhile to notice that experiments are always in excellent agreement with our predictions, despite using data distributions that are out of the hypotheses of~\autoref{th:scaling}. Indeed, for regression with the actual NTK even the assumption of translationally-invariant constituents is violated. Moreover, we report the learning curves of local kernels on the CIFAR-10 dataset showing that smaller filter sizes correspond to faster decays even for real and anisotropic data distributions, in agreement with the picture emerging from our synthetic model.

\begin{figure}
    \centering
    \includegraphics[width=1.0\linewidth]{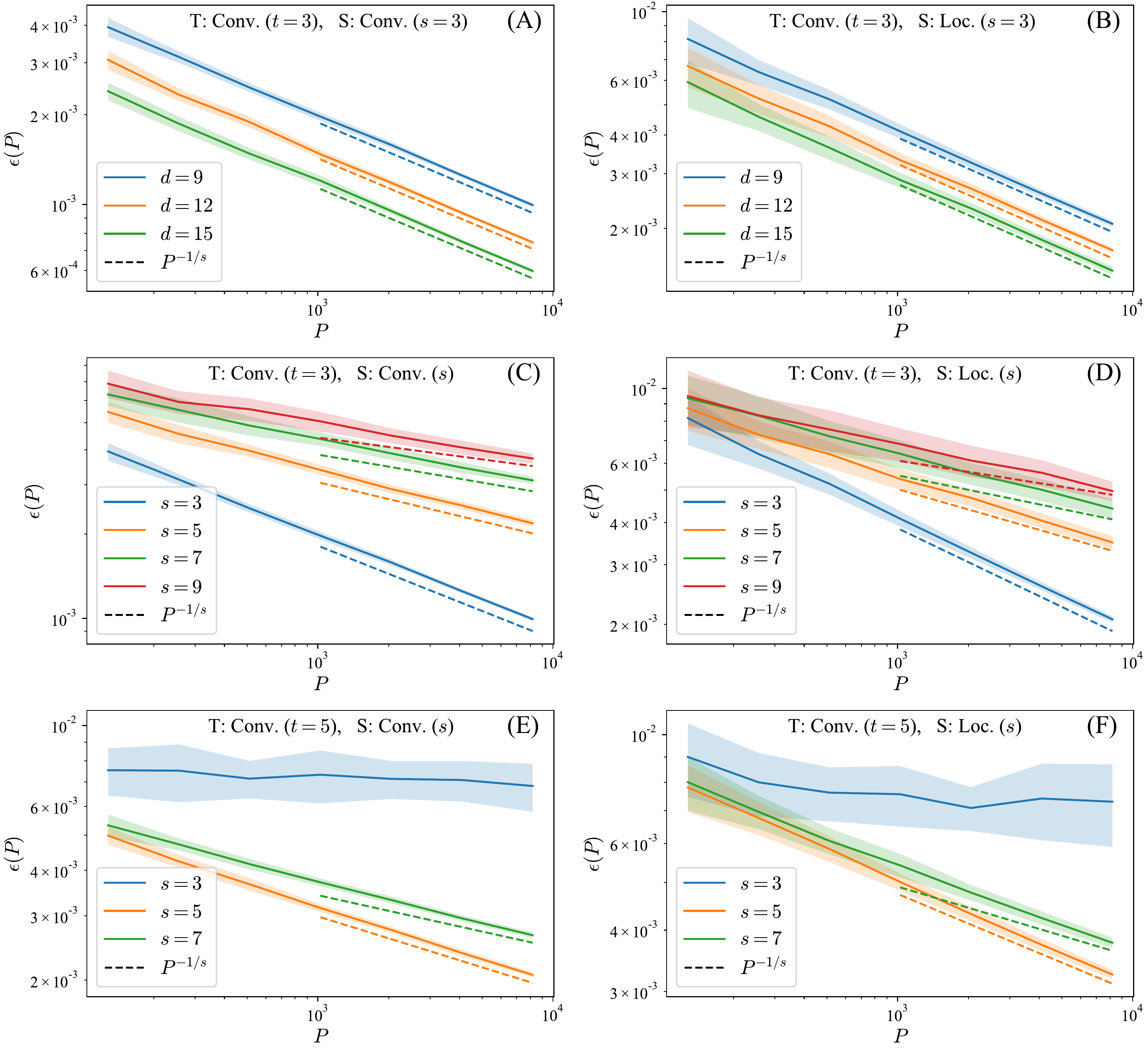}
    \caption{Learning curves for different combinations of convolutional teachers with convolutional (left panels) and local (right panels) students. The teacher and student filter sizes are denoted with $t$ and $s$ respectively. Data are sampled uniformly in the hypercube $[0,1]^d$, with $d=9$ if not specified otherwise. Solid lines are the results of numerical experiments averaged over 128 realisations and the shaded areas represent the empirical standard deviations. The predicted scalings are shown by dashed lines. All the panels are discussed in \autoref{sec:empirical}, while additional details on experiments are reported in \autoref{app:numerics}, together with additional experiments.}
    \label{fig:figure}
\end{figure}

\section{Asymptotics of learning curves with decreasing ridge}

We now prove an upper bound for the exponent $\beta$ implying that the curse of dimensionality is beaten by a local or convolutional kernel learning a convolutional target (as in~\autoref{eq:loc-f}), using the framework developed in~\cite{jacot2020kernel} and a natural universality assumption on the kernel eigenfunctions. It is worth noticing that this framework does not require the target function to be generated by a teacher kernel. Proofs are presented in~\autoref{app:thm2}. 
Let $\mathcal{D}(\Lambda)$ denote the density of eigenvalues of the student kernel, $\mathcal{D}(\Lambda) \,{=}\, \sum_\rho \delta(\Lambda - \Lambda_\rho)$, with $\delta(x)$ denoting Dirac delta function. Having a random target function with coefficients $c_\rho$ in the kernel eigenbasis having variance $\mathbb{E}[|c_\rho|^2]$, one can define the following reduced density (with respect to the teacher):
\begin{equation}\label{eq:reduced-density}
\mathcal{D}_T(\Lambda) = \displaystyle\sum_{\left\lbrace\rho| \mathbb{E}[|c_\rho|^2] > 0\right\rbrace } \delta(\Lambda - \Lambda_\rho)
\end{equation}
$\mathcal{D}_T(\Lambda)$ counts  eigenvalues for which the target has a non-zero variance, such that:
\begin{equation}
    \sum_{\rho} \mathbb{E}[|c_\rho|^2] = \int d\Lambda \mathcal{D}_T(\Lambda) c^2(\Lambda),
\end{equation}
where the function $c(\Lambda)$ is defined by $c^2(\Lambda_\rho)\,{=}\,\mathbb{E}[|c_\rho|^2] $ for all $\rho$ such that $\mathbb{E}[|c_\rho|^2]\,{>}\,0$.
The following theorem then follows from the results of~\cite{jacot2020kernel}. 
\begin{theorem}\label{th:ridge}
Let us consider a positive-definite kernel $K$ with eigenvalues $\Lambda_\rho$, $\sum_\rho \Lambda_\rho < \infty$, and eigenfunctions ${\Phi_\rho}$ learning a (random) target function $f^*$ in kernel ridge regression (\autoref{eq:argmin}) with ridge $\lambda$ from $P$ observations $f^*(\bm{x}^\mu)$, with $\bm{x}^\mu\in \mathbb{R}^d$ drawn from a certain probability distribution. Let us denote with $\mathcal{D}_T(\Lambda)$ the reduced density of kernel eigenvalues with respect to the target and $\epsilon(\lambda,P)$ the generalisation error and also assume that
\begin{itemize}
    \item[$i)$] For any $P$-tuple of indices $\rho_1,\dots,\rho_P$, the vector $(\Phi_{\rho_1}(\bm{x}^1), \dots,\Phi_{\rho_P}(\bm{x}^P))$ is a Gaussian random vector;
    \item[$ii)$] The target function can be written in the kernel eigenbasis with coefficients $c_\rho$ and $c^2(\Lambda_\rho)\,{=}\,\mathbb{E}[|c_\rho|^2]$, with $\mathcal{D}_T(\Lambda) \sim \Lambda^{-(1+r)}$, $c^2(\Lambda) \sim \Lambda^{q}$ asymptotically for small $\Lambda$ and $r\,{>}\,0$, $ r\,{<}\,q\,{<}\,r+2$;
\end{itemize}
Then the following equivalence holds in the joint $P\to\infty$ and $\lambda\to 0$ limit with $1/(\lambda\sqrt{P})\to 0$:
\begin{equation}
    \epsilon(\lambda, P) \sim  \sum_{\left\lbrace \rho|\Lambda_\rho < \lambda \right\rbrace} \mathbb{E}{[|c_\rho|^2]} = \int_0^{\lambda} d\Lambda \mathcal{D}_T(\Lambda) c^2(\Lambda).
\end{equation}
\end{theorem}

Note that the assumption $i)$ of the theorem on the Gaussianity of the eigenbasis does not hold in our setup where  the $\Phi_\rho$'s are plane waves. However,  the random variables $\Phi_{\rho}(\bm{x}^\mu)$ have a probability density  with compact support. It is thus natural to assume that a Gaussian universality assumption holds, i.e. that~\autoref{th:ridge} applies to our problem. With this assumption, we obtain the following
\begin{corollary}\label{cor:beta-rigorous}
Performing kernel ridge regression in a teacher-student scenario with smoothness exponents $\alpha_t$ (teacher) and $\alpha_s$ (student), with ridge $\lambda\sim P^{-\gamma}$ and $0\,{<}\,\gamma\,{<}\,1/2$, under the joint hypotheses of~\autoref{th:scaling} and ~\autoref{th:ridge}, the exponent governing the asymptotic scaling of the generalisation error with $P$ is given by: 
\begin{equation}\label{eq:beta-rigorous}
    \beta = \gamma\frac{\alpha_t}{\alpha_s + s},
\end{equation}
\end{corollary}
which does not vanish in the limit $d\rightarrow\infty$. Furthermore, ~\autoref{eq:beta-rigorous} depends on $s$ and not on $t$ as the prediction of~\autoref{eq:prediction}.

\section{Conclusions and future work}

Our work shows that, even in large  dimension $d$, a function can be learnt efficiently if it
can be expressed as a sum of constituent functions each depending on a smaller number of variables $t$, by performing regression with a kernel that entails such a  compositional structure with  $s$-dimensional constituents.  The learning curve exponent is then independent of $d$ and governed by $s$ if $s\geq t$, optimal for $s\,{=}\,t$ and null if $s\,{<}\,t$. 

In the context of image classification, this result relates to the ``Bag of Words'' viewpoint. Consider for example  two-dimensional images consisting of $M$ features of $t$ adjacent pixels, and that different classes correspond to distinct subsets of (possibly shared) features. If features can be located anywhere, then data lie on a $2M$-dimensional manifold. On the one hand, we expect a one-hidden-layer convolutional network with filter size $s\,{\geq}\,t$ to learn well with a learning curve exponent governed by $s$ and independent of $M$. On the other hand, a fully-connected network would suffer from the curse of dimensionality for large $M$.

Our work does not consider that the compositional structure of real data is hierarchical, with large features that consist of smaller sub-features. It is intuitively clear that depth and locality taken together are well-suited for such data structure \cite{bietti2021approximation,poggio2017and}. Extending the present teacher-student framework to this case would offer valuable quantitative insights into the question of how many data are required to learn such tasks.

\section*{Acknowledgements}
We thank Alberto Bietti, Stefano Spigler, Antonio Sclocchi, Leonardo Petrini, Mario Geiger, and Umberto Maria Tomasini for helpful discussions. This work was supported by a grant from the Simons Foundation (\#454953 Matthieu Wyart).

\bibliography{main}

\begin{thebibliography}{10}

\bibitem{luxburg2004distance}
Ulrike~von Luxburg and Olivier Bousquet.
\newblock Distance-based classification with lipschitz functions.
\newblock {\em Journal of Machine Learning Research}, 5(Jun):669--695, 2004.

\bibitem{hestness2017deep}
Joel Hestness, Sharan Narang, Newsha Ardalani, Gregory Diamos, Heewoo Jun,
  Hassan Kianinejad, Md~Patwary, Mostofa Ali, Yang Yang, and Yanqi Zhou.
\newblock Deep learning scaling is predictable, empirically.
\newblock {\em arXiv preprint arXiv:1712.00409}, 2017.

\bibitem{krizhevsky2012imagenet}
Alex Krizhevsky, Ilya Sutskever, and Geoffrey~E Hinton.
\newblock Imagenet classification with deep convolutional neural networks.
\newblock {\em Advances in neural information processing systems},
  25:1097--1105, 2012.

\bibitem{spigler2019asymptotic}
Stefano Spigler, Mario Geiger, and Matthieu Wyart.
\newblock Asymptotic learning curves of kernel methods: empirical data versus
  teacher–student paradigm.
\newblock {\em Journal of Statistical Mechanics: Theory and Experiment},
  2020(12):124001, December 2020.
\newblock Publisher: IOP Publishing.

\bibitem{biederman1987recognition}
Irving Biederman.
\newblock Recognition-by-components: a theory of human image understanding.
\newblock {\em Psychological review}, 94(2):115, 1987.

\bibitem{poggio2017and}
Tomaso Poggio, Hrushikesh Mhaskar, Lorenzo Rosasco, Brando Miranda, and Qianli
  Liao.
\newblock Why and when can deep-but not shallow-networks avoid the curse of
  dimensionality: a review.
\newblock {\em International Journal of Automation and Computing},
  14(5):503--519, 2017.

\bibitem{deza2020hierarchically}
Arturo Deza, Qianli Liao, Andrzej Banburski, and Tomaso Poggio.
\newblock Hierarchically compositional tasks and deep convolutional networks,
  2020.

\bibitem{bietti2021approximation}
Alberto Bietti.
\newblock On approximation in deep convolutional networks: a kernel
  perspective, 2021.

\bibitem{Neyshabur2020towards}
Behnam Neyshabur.
\newblock Towards learning convolutions from scratch.
\newblock In H.~Larochelle, M.~Ranzato, R.~Hadsell, M.~F. Balcan, and H.~Lin,
  editors, {\em Advances in Neural Information Processing Systems}, volume~33,
  pages 8078--8088. Curran Associates, Inc., 2020.

\bibitem{fukushima1975cognitron}
Kunihiko Fukushima.
\newblock Cognitron: A self-organizing multilayered neural network.
\newblock {\em Biological cybernetics}, 20(3):121--136, 1975.

\bibitem{lecun1989generalization}
Yann LeCun et~al.
\newblock Generalization and network design strategies.
\newblock {\em Connectionism in perspective}, 19:143--155, 1989.

\bibitem{jacot2018neural}
Arthur Jacot, Franck Gabriel, and Cl{\'e}ment Hongler.
\newblock Neural tangent kernel: Convergence and generalization in neural
  networks.
\newblock In {\em Proceedings of the 32Nd International Conference on Neural
  Information Processing Systems}, NIPS'18, pages 8580--8589, USA, 2018. Curran
  Associates Inc.

\bibitem{Du2019}
Simon~S. Du, Xiyu Zhai, Barnabas Poczos, and Aarti Singh.
\newblock Gradient descent provably optimizes over-parameterized neural
  networks.
\newblock In {\em International Conference on Learning Representations}, 2019.

\bibitem{lee2019wide}
Jaehoon Lee, Lechao Xiao, Samuel Schoenholz, Yasaman Bahri, Roman Novak, Jascha
  Sohl-Dickstein, and Jeffrey Pennington.
\newblock Wide {Neural} {Networks} of {Any} {Depth} {Evolve} as {Linear}
  {Models} {Under} {Gradient} {Descent}.
\newblock In {\em Advances in {Neural} {Information} {Processing} {Systems}
  32}, pages 8572--8583. Curran Associates, Inc., 2019.

\bibitem{arora2019exact}
Sanjeev Arora, Simon~S Du, Wei Hu, Zhiyuan Li, Russ~R Salakhutdinov, and
  Ruosong Wang.
\newblock On {Exact} {Computation} with an {Infinitely} {Wide} {Neural} {Net}.
\newblock In {\em Advances in {Neural} {Information} {Processing} {Systems}
  32}, pages 8141--8150. Curran Associates, Inc., 2019.

\bibitem{chizat2019lazy}
Lenaic Chizat, Edouard Oyallon, and Francis Bach.
\newblock On lazy training in differentiable programming.
\newblock In {\em Advances in Neural Information Processing Systems}, pages
  2937--2947, 2019.

\bibitem{bordelon2020spectrum}
Blake Bordelon, Abdulkadir Canatar, and Cengiz Pehlevan.
\newblock Spectrum {Dependent} {Learning} {Curves} in {Kernel} {Regression} and
  {Wide} {Neural} {Networks}.
\newblock In {\em International {Conference} on {Machine} {Learning}}, pages
  1024--1034. PMLR, November 2020.
\newblock ISSN: 2640-3498.

\bibitem{canatar2021spectral}
Abdulkadir Canatar, Blake Bordelon, and Cengiz Pehlevan.
\newblock Spectral bias and task-model alignment explain generalization in
  kernel regression and infinitely wide neural networks.
\newblock {\em Nature Communications}, 12(1):1--12, 2021.

\bibitem{loureiro2021capturing}
Bruno Loureiro, C{\'e}dric Gerbelot, Hugo Cui, Sebastian Goldt, Florent
  Krzakala, Marc M{\'e}zard, and Lenka Zdeborov{\'a}.
\newblock Capturing the learning curves of generic features maps for realistic
  data sets with a teacher-student model.
\newblock {\em arXiv preprint arXiv:2102.08127}, 2021.

\bibitem{micchelli1979design}
Charles~A Micchelli and Grace Wahba.
\newblock Design problems for optimal surface interpolation.
\newblock Technical report, WISCONSIN UNIV-MADISON DEPT OF STATISTICS, 1979.

\bibitem{jacot2020kernel}
Arthur Jacot, Berfin Simsek, Francesco Spadaro, Clement Hongler, and Franck
  Gabriel.
\newblock Kernel alignment risk estimator: Risk prediction from training data.
\newblock In H.~Larochelle, M.~Ranzato, R.~Hadsell, M.~F. Balcan, and H.~Lin,
  editors, {\em Advances in Neural Information Processing Systems}, volume~33,
  pages 15568--15578. Curran Associates, Inc., 2020.

\bibitem{kondor2018generalization}
Risi Kondor and Shubhendu Trivedi.
\newblock On the generalization of equivariance and convolution in neural
  networks to the action of compact groups.
\newblock In {\em International Conference on Machine Learning}, pages
  2747--2755. PMLR, 2018.

\bibitem{zhou2018building}
H.~H. Zhou, Y.~Xiong, and V.~Singh.
\newblock Building bayesian neural networks with blocks: On structure,
  interpretability and uncertainty.
\newblock {\em arXiv preprint, arXiv:1806.03563}, 2018.

\bibitem{Poggio2020theo}
Tommaso Poggio, Andrzej Banburski, and Qianli Liao.
\newblock Theoretical issues in deep networks.
\newblock {\em Proceedings of the National Academy of Sciences},
  117(48):30039--30045, 2020.

\bibitem{malach2021computational}
Eran Malach and Shai Shalev-Shwartz.
\newblock Computational separation between convolutional and fully-connected
  networks.
\newblock In {\em International Conference on Learning Representations}, 2021.

\bibitem{mairal2016end}
Julien Mairal.
\newblock End-to-end kernel learning with supervised convolutional kernel
  networks.
\newblock {\em arXiv preprint arXiv:1605.06265}, 2016.

\bibitem{bietti2019inductive}
Alberto Bietti and Julien Mairal.
\newblock On the inductive bias of neural tangent kernels.
\newblock {\em arXiv preprint arXiv:1905.12173}, 2019.

\bibitem{mei2021learning}
Song Mei, Theodor Misiakiewicz, and Andrea Montanari.
\newblock Learning with invariances in random features and kernel models, 2021.

\bibitem{kanagawa2018gaussian}
Motonobu Kanagawa, Philipp Hennig, Dino Sejdinovic, and Bharath~K
  Sriperumbudur.
\newblock Gaussian processes and kernel methods: A review on connections and
  equivalences.
\newblock {\em arXiv preprint arXiv:1807.02582}, 2018.

\bibitem{bach2017breaking}
Francis Bach.
\newblock Breaking the curse of dimensionality with convex neural networks.
\newblock {\em The Journal of Machine Learning Research}, 18(1):629--681, 2017.

\bibitem{paccolat2020isotropic}
Jonas Paccolat, Stefano Spigler, and Matthieu Wyart.
\newblock How isotropic kernels perform on simple invariants.
\newblock {\em Machine Learning: Science and Technology}, 2(2):025020, March
  2021.
\newblock Publisher: IOP Publishing.

\bibitem{geifman2020similarity}
Amnon Geifman, Abhay Yadav, Yoni Kasten, Meirav Galun, David Jacobs, and Basri
  Ronen.
\newblock On the similarity between the laplace and neural tangent kernels.
\newblock In H.~Larochelle, M.~Ranzato, R.~Hadsell, M.~F. Balcan, and H.~Lin,
  editors, {\em Advances in Neural Information Processing Systems}, volume~33,
  pages 1451--1461. Curran Associates, Inc., 2020.

\bibitem{Neal1996}
Radford~M. Neal.
\newblock {\em Bayesian Learning for Neural Networks}.
\newblock Springer-Verlag New York, Inc., Secaucus, NJ, USA, 1996.

\bibitem{williams1997computing}
Christopher~KI Williams.
\newblock Computing with infinite networks.
\newblock {\em Advances in neural information processing systems}, pages
  295--301, 1997.

\bibitem{lee2017deep}
Jaehoon Lee, Yasaman Bahri, Roman Novak, Samuel~S Schoenholz, Jeffrey
  Pennington, and Jascha Sohl-Dickstein.
\newblock Deep neural networks as gaussian processes.
\newblock {\em arXiv preprint arXiv:1711.00165}, 2017.

\bibitem{cho2009kernel}
Youngmin Cho and Lawrence Saul.
\newblock Kernel methods for deep learning.
\newblock In Y.~Bengio, D.~Schuurmans, J.~Lafferty, C.~Williams, and
  A.~Culotta, editors, {\em Advances in Neural Information Processing Systems},
  volume~22. Curran Associates, Inc., 2009.

\bibitem{scholkopf2001learning}
Bernhard Scholkopf and Alexander~J Smola.
\newblock {\em Learning with kernels: support vector machines, regularization,
  optimization, and beyond}.
\newblock MIT press, 2001.

\bibitem{widom1964asymptotic}
Harold Widom.
\newblock Asymptotic behavior of the eigenvalues of certain integral equations.
  ii.
\newblock {\em Archive for Rational Mechanics and Analysis}, 17(3):215--229,
  1964.

\bibitem{Mezard87}
Marc {M{\'e}zard}, Giorgio Parisi, and Miguel Virasoro.
\newblock {\em Spin glass theory and beyond: An Introduction to the Replica
  Method and Its Applications}, volume~9.
\newblock World Scientific Publishing Company, 1987.

\bibitem{potters2020first}
Marc Potters and Jean-Philippe Bouchaud.
\newblock {\em A First Course in Random Matrix Theory: For Physicists,
  Engineers and Data Scientists}.
\newblock Cambridge University Press, 2020.

\bibitem{jacot2020implicit}
Arthur Jacot, Berfin Simsek, Francesco Spadaro, Cl\'ement Hongler, and Franck
  Gabriel.
\newblock Implicit regularization of random feature models.
\newblock In {\em International Conference on Machine Learning}, pages
  4631--4640, 2020.

\end{thebibliography}
\bibliographystyle{unsrt}

\newpage
\appendix

\renewcommand{\thefigure}{S\arabic{figure}}
\setcounter{figure}{0}

\renewcommand{\theequation}{S\arabic{equation}}
\setcounter{equation}{0}



\renewcommand\contentsname{Supplementary Material}

\tableofcontents

\addtocontents{toc}{\setcounter{tocdepth}{1}}

\section{Spectral bias in kernel regression}\label{app:spectral}

In this appendix we provide additional details about the derivation of~\autoref{eq:error-scaling} within the framework of~\cite{bordelon2020spectrum, canatar2021spectral}. Let us begin by recalling the definition of the kernel ridge regression estimator $f$ of a target function $f^*$,
\begin{equation}\label{eq:argmin-app}
 f =\argmin_{f\in \mathcal{H}} \left\lbrace \frac{1}{P}\displaystyle \sum_{\mu=1}^P \left(f(\bm{x}^\mu) - f^*(\bm{x}^\mu)\right)^2  + \lambda \, \|f\|^2_{\mathcal{H}} \right  \rbrace,
\end{equation}
where $\mathcal{H}$ denotes the Reproducing Kernel Hilbert Space (RKHS) of the kernel $K(\bm{x},\bm{y})$. After introducing the Mercer's decomposition of the kernel,
\begin{equation}\label{eq:mercer-app}
 K(\bm{x},\bm{y}) = \sum_{\rho=1}^\infty \lambda_\rho \phi_\rho(\bm{x}) \overline{\phi_\rho(\bm{y})}, \quad \int p\left( d^d y\right) K(\bm{x},\bm{y}) \phi_\rho(\bm{y}) = \lambda_\rho \phi_\rho(\bm{x}). 
\end{equation}
the RKHS can be characterised as a subset of the space of functions lying in the span of the kernel eigenbasis,
\begin{equation}\label{eq:RKHS-app}
    \mathcal{H} = \left\lbrace f = \sum_{\rho=1}^\infty a_\rho \phi_\rho(\bm{x})\right. \left| \; \sum_{\rho=1}^\infty \frac{|a_\rho|^2}{\lambda_\rho} <\infty \right\rbrace.
\end{equation}
In other words, the RKHS contains functions having a finite norm $||f||_\mathcal{H} \eq \sqrt{\left\langle f, f\right\rangle_{\mathcal{H}}}$ with respect to the following inner product,
\begin{equation}
f(\bm{x})=\sum_\rho a_\rho \phi_\rho(\bm{x}),\; f'(\bm{x})=\sum_\rho a'_\rho \phi_\rho(\bm{x}), \; \left\langle f, f'\right\rangle_{\mathcal{H}} = \sum_\rho \frac{a_\rho a'_\rho}{\lambda_\rho}.
\end{equation}
Given any target function $f^*$ lying in the span of the kernel eigenbasis, the mean squared generalisation error of the kernel ridge regression estimator reads
\begin{equation}\label{eq:error-app}
    \epsilon(\lambda, \left\lbrace \bm{x}^\mu\right\rbrace) = \int p(d^d\bm{x}) \left(f(\bm{x})-f^*(\bm{x})\right)^2 = \sum_{\rho=1}^\infty \left| a_\rho(\lambda, \left\lbrace \bm{x}^\mu\right\rbrace) -c_\rho\right|^2,
\end{equation}
with $c_\rho$ denoting the $\rho$-th coefficient of the target $f^*$ and $a_\rho$ that of the estimator $f$, which depends on the ridge $\lambda$ and on the training set $\left\lbrace \bm{x}^\mu\right\rbrace_{\mu=1,\dots,P}$. Notice that, as $f$ belongs to $\mathcal{H}$ by definition, $\sum_\rho |a_\rho|^2/\lambda_\rho \,{<}\,+\infty$, whereas the $c_\rho$'s are free to take any value.

The authors of~\cite{bordelon2020spectrum, canatar2021spectral} found a heuristic expression for the average of the mean squared error over realisations of the training set $\left\lbrace \bm{x}^\mu\right\rbrace$. Such expression, based on the replica method of statistical physics, reads\footnote{Notice that the risk considered in \cite{bordelon2020spectrum, canatar2021spectral} slightly differs from \autoref{eq:argmin-app} by a factor $1/P$ in front of the sum.}
\begin{equation}\label{eq:bordelon1}\epsilon(\lambda, P) = \partial_{\lambda}\left(\frac{\kappa_{\lambda}(P)}{P}\right) \sum_\rho \frac{\kappa_\lambda(P)^2}{\left(P \lambda_\rho + \kappa_\lambda(P)\right)^2} |c_\rho|^2 ,\end{equation}
where $\kappa(P)$ satisfies
\begin{equation}\label{eq:bordelon2} \frac{\kappa_\lambda(P)}{P} = \lambda + \frac{1}{P}\sum_\rho \frac{\lambda_\rho \kappa_\lambda(P)/P}{\lambda_\rho + \kappa_\lambda(P)/P}.
\end{equation}
In short, the replica method works as follows~\cite{Mezard87}: first one defines an energy function $E(f)$ as the argument of the minimum in~\autoref{eq:argmin-app}, then attribute to the predictor $f$ a Boltzmann-like probability distribution $P(f) = Z^{-1} e^{-\beta E(f)}$ , with $Z$ a normalisation constant and $\beta\,{>}\,0$. As $\beta\to\infty$, the probability distribution $P(f)$ concentrates around the solution of the minimisation problem of~\autoref{eq:argmin-app}, i.e. the predictor of kernel regression. Hence, one can replace $f$ in the right-hand side of~\autoref{eq:error-app} with an average over $P(f)$ at finite $\beta$, then perform the limit $\beta\to\infty$ after the calculation so as to recover the generalisation error. The simplification stems from the fact that, once $f$ is replaced with its eigendecomposition, the energy function $E(f)$ becomes a quadratic function of the coefficients $c_\rho$. Then, under the assumption that the data distribution enters only via the first and second moments of the eigenfunctions $\phi_{\rho}(\bm{x})$ w.r.t $\bm{x}$, all averages in~\autoref{eq:error-app} reduce to Gaussian integrals.

Mathematically, $\kappa_\lambda(P)/P$ is related to the Stieltjes transform~\cite{potters2020first} of the Gram matrix $\mathbb{K}_P/P$ in the large-$P$ limit. Intuitively, it plays the role of a threshold: the modal contributions to the error tend to $0$ for $\rho$ such that $\lambda_\rho \gg k_\lambda(P)/P$, and to $\mathbb{E}[|c_\rho|^2]$ for $\rho$ such that $\lambda_\rho \ll k_\lambda(P)/P$. This is equivalent to saying that the algorithm predictor $f(\bm{x})$ captures only the eigenmodes having eigenvalue larger than $k_\lambda(P)/P$ (see also~\cite{jacot2020implicit, jacot2020kernel}). 

This intuitive picture can actually be exploited in order to extract the learning curve exponent $\beta$ from the asymptotic behaviour of~\autoref{eq:bordelon1} and~\autoref{eq:bordelon2} in the ridgeless limit $\lambda\to 0^+$. In the following, we assume that both the kernel and the target function have a power-law spectrum, in particular $\lambda_\rho \sim \rho^{-a}$ and $\mathbb{E}[|{c^*_\rho}|^2] \sim \rho^{-b}$, with $2a\,{>}\,b-1$. First, we approximate the sum over modes in \autoref{eq:bordelon2} with an integral using the Euler-Maclaurin formula. Then we substitute the eigenvalues inside the integral with their asymptotic limit, $\lambda_\rho = A\rho^{-a}$. Since, $\kappa_0(P)/P \to 0$ as $P \to \infty$, both these operations result in an error which is asymptotically independent of $P$. Namely,
\begin{align}
    \frac{\kappa_0(P)}{P} &= \frac{\kappa_0(P)}{P} \frac{1}{P} \left(\int_0^\infty \frac{d\rho \, A\rho^{-a} }{A\rho^{-a} + \kappa_0(P)/P} + \mathcal{O}(1) \right) \\
    &= \frac{\kappa_0(P)}{P} \frac{1}{P} \left( \left( \frac{\kappa_0(P)}{P} \right)^{-\frac{1}{a}}\int_0^\infty \frac{d\sigma \, \sigma^{\frac{1}{a}-1}A^{\frac{1}{a}}a^{-1}}{1 + \sigma} + \mathcal{O}(1) \right), \nonumber
\end{align}
where in the second line, we changed the integration variable from $\rho$ to $\sigma\,{=}\, \kappa_0(P)\rho^a/( A P)$. Since the integral in $\sigma$ is finite and independent of $P$, we obtain that $\kappa_0(P)/P = \mathcal{O}(P^{-a})$. Similarly, we find that the mode-independent prefactor $\partial_\lambda \left(\kappa_\lambda(P)/P\right)|_{\lambda=0} = \mathcal{O}(1)$.
As a result we are left with, modulo some $P$-independent prefactors,
\begin{equation}\label{eq:error-scaling1} \epsilon(P) \sim \sum_{\rho} \frac{P^{-2a}}{\left(A\rho^{-a}+P^{-a}\right)^2}\mathbb{E}[|c_\rho|^2].\end{equation}
Following the intuitive argument about the thresholding role of $\kappa_0(P)/P \sim P^{-a}$, it is convenient to split the sum in~\autoref{eq:error-scaling1} into sectors where $\lambda_\rho\gg\kappa_0(P)/P $, $\lambda_\rho \sim \kappa_0(P)/P $ and $\lambda_\rho\ll\kappa_0(P)/P $, i.e.,
\begin{equation}\label{eq:error-scaling2} \epsilon(P) \sim \sum_{\rho \ll P} \frac{P^{-2a}}{\left(A\rho^{-a}\right)^2}\mathbb{E}[|c_\rho|^2] +\sum_{\rho \sim P} \frac{1}{2}\mathbb{E}[|c_\rho|^2] + \sum_{\rho \gg P} \mathbb{E}[|c_\rho|^2].\end{equation}
Finally, ~\autoref{eq:error-scaling} is obtained by noticing that, under our assumptions on the decay of $\mathbb{E}[|c_\rho|^2]$ with $\rho$, the contribution of the sum over $\rho\ll P$ is subleading in $P$ whereas the other two sums can be gathered together.

\section{NTKs of convolutional and locally-connected networks}\label{app:ntk}

We begin this section by reviewing the computation of the NTK of a one-hidden-layer fully-connected network \cite{chizat2019lazy}.

\begin{definition}[one-hidden-layer FCN]\label{eq:fcn}
A one-hidden-layer fully-connected network with $H$ hidden neurons is defined as follows,
\begin{equation}
    f^{FCN}(\bm{x}) = \frac{1}{\sqrt{H}} \sum_{h=1}^H a_h \sigma(\bm{w}_{h} \cdot \bm{x} + b_h),
\end{equation}
where $\bm{x}\in\mathbb{R}^d$ is the input, $H$ is the width, $\sigma$ is a nonlinear activation function, $\{\bm{w}_{h}\in\mathbb{R}^d\}_{h=1}^H$, $\{b_{h}\in\mathbb{R}\}_{h=1}^H$, and $\{a_{h}\in\mathbb{R}\}_{h=1}^H$ are the network's parameters. The dot $\cdot$ denotes the standard Euclidean scalar product.
\end{definition}

Inserting (\ref{eq:fcn}) into Definition~\ref{eq:finite_ntk}, one obtains

\begin{align}\label{eq:rand-feat-fc}
    \Theta^{FC}_N(\bm{x},\bm{y};\bm{\theta}) = \frac{1}{H} \sum_{h=1}^H &\left( \sigma(\bm{w}_h \cdot \bm{x} + b_h) \sigma (\bm{w}_h \cdot \bm{y} + b_h) \right. \\
    &+ \left. a_h^2 \sigma'(\bm{w}_h \cdot \bm{x} + b_h) \sigma'(\bm{w}_h \cdot \bm{y} + b_h) (\bm{x} \cdot \bm{y} + 1) \right), \nonumber
\end{align}

where $\sigma'$ denotes the derivative of $\sigma$ with respect to its argument. If all the parameters are initialised independently from a standard Normal distribution, $\Theta^{FC}_N(\bm{x},\bm{y};\bm{\theta})$ is a random-feature kernel that in the $H \to \infty$ limit converges to

\begin{align}\label{eq:fc-ntk}
    \Theta^{FC}(\bm{x},\bm{y}) &=  \mathbb{E}_{\bm{w},b}[\sigma(\bm{w} \cdot \bm{x} + b) \sigma (\bm{w} \cdot \bm{y} + b)] \\ 
    &+ \mathbb{E}_{a}[a^2] \mathbb{E}_{\bm{w},b}[\sigma'(\bm{w} \cdot \bm{x} + b) \sigma'(\bm{w} \cdot \bm{y} + b)] (\bm{x} \cdot \bm{y} + 1). \nonumber
\end{align}

When $\sigma$ is the ReLU activation function, the expectations can be computed exactly using techniques from the literature of arc-cosine kernels \cite{cho2009kernel}

\begin{align}\label{eq:relu-fc-ntk}
	\Theta^{FC}(\bm{x},\bm{y}) &= \frac{1}{2\pi} \sqrt{\|\bm{x}\|^2+1} \sqrt{\|\bm{y}\|^2+1} \, (\sin\varphi + (\pi-\varphi)\cos\varphi) \\ &+ \frac{1}{2\pi} (\bm{x} \cdot \bm{y} + 1) (\pi-\varphi), \nonumber
\end{align}

with $\varphi$ denoting the angle

\begin{equation}
	\varphi = \arccos \left( \frac{\bm{x} \cdot \bm{y} + 1}{\sqrt{\|\bm{x}\|^2+1} \sqrt{\|\bm{y}\|^2+1}}\right).
\end{equation}

Notice that, as commented in \autoref{sec:convolutional-mercer}, for ReLU networks $\Theta^{FC}(\bm{x},\bm{y})$ displays a cusp at $\bm{x} \eq \bm{y}$.

\paragraph{Proof of Lemma \ref{lemma:cntk}}

\begin{proof}

Inserting \autoref{eq:cnn-out} into Definition~\ref{eq:finite_ntk},
\begin{align}\label{eq:rand-feat-cn}
    \Theta^{CN}_N(\bm{x},\bm{y};\bm{\theta}) =  \frac{1}{|\mathcal{P}|^2} \sum_{i,j\in\mathcal{P}} \Biggl( \frac{1}{H} \sum_{h=1}^H &\bigl( \sigma(\bm{w}_h \cdot \bm{x}_i + b_h) \sigma (\bm{w}_h \cdot \bm{y}_j + b_h) \\ &+ a_h^2 \sigma'(\bm{w}_h \cdot \bm{x}_i + b_h) \sigma'(\bm{w}_h \cdot \bm{y}_j + b_h) (\bm{x}_i \cdot \bm{y}_j + 1) \bigr) \Biggr) \nonumber
\end{align}

In the previous line, the single terms of the summation over patches are the random-feature kernels $\Theta_N^{FC}$ obtained in \autoref{eq:rand-feat-fc} acting on $s$-dimensional inputs, i.e. the patches of $\bm{x}$ and $\bm{y}$. Therefore,

\begin{equation}\label{eq:conv-ntk-finite}
    \Theta^{CN}_N(\bm{x},\bm{y};\bm{\theta}) =  \frac{1}{|\mathcal{P}|^2} \sum_{i,j\in\mathcal{P}} \Theta^{(FC)}_N(\bm{x},\bm{y}).
\end{equation}

If all the parameters are initialised independently from a standard Normal distribution, the $H \to \infty$ limit of~\autoref{eq:conv-ntk-finite} yields \autoref{eq:conv-ntk}.

\end{proof}

\paragraph{Proof of Lemma \ref{lemma:lntk}}

\begin{proof}

Inserting \autoref{eq:lcn-out} into Definition~\ref{eq:finite_ntk},

\begin{align}\label{eq:rand-feat-lc}
    \Theta^{LC}_N(\bm{x},\bm{y};\bm{\theta}) = \frac{1}{|\mathcal{P}|} \sum_{i\in\mathcal{P}} \Biggl( \frac{1}{H} \sum_{h=1}^H  &\bigl( \sigma(\bm{w}_{h,i} \cdot \bm{x}_i + b_{h,i}) \sigma (\bm{w}_{h,i} \cdot \bm{y}_i + b_{h,i})
    \\ &+ a_{h,i}^2 \sigma'(\bm{w}_{h,i} \cdot \bm{x}_i + b_{h,i}) \sigma'(\bm{w}_{h,i} \cdot \bm{y}_i + b_{h,i}) (\bm{x}_i \cdot \bm{y}_i + 1) \bigl) \Biggl). \nonumber
\end{align}

In the previous line, the single terms of the summation over patches are the random-feature kernels $\Theta_N^{FC}$ obtained in \autoref{eq:rand-feat-fc} acting on $s$-dimensional inputs, i.e. the patches of $\bm{x}$ and $\bm{y}$. Therefore, 

\begin{equation}
    \Theta^{LC}_N(\bm{x},\bm{y};\bm{\theta}) =  \frac{1}{|\mathcal{P}|} \sum_{i\in\mathcal{P}} \Theta^{(FC)}_N(\bm{x}_i,\bm{y}_i).
\end{equation}

If all the parameters are initialised independently from a standard Normal distribution, \autoref{eq:loc-ntk} is recovered in the $H\to\infty$ limit.

\end{proof}

\section{Mercer's decomposition of convolutional and local kernels}\label{app:mercer-overlap}

In this section we prove the eigendecompositions introduced in Lemma~\ref{lemma:conv-spectra} and Lemma~\ref{lemma:loc-spectra}, then extend them to overlapping-patches kernel (cf.~\ref{asec:mercer-overlapping}). We define the scalar product in input space between two (complex) functions $f$ and $g$ as

\begin{equation}
    \left\langle f, g\right\rangle = \int p(d^dx)\, f(\bm{x}) \overline{g(\bm{x})}.
\end{equation}

\paragraph{Proof of Lemma \ref{lemma:conv-spectra}}

\begin{proof}

We start by proving orthonormality of the eigenfunctions. By writing the $d$-dimensional eigenfunctions $\Phi_\rho$ in terms of the $s$-dimensional eigenfunctions $\phi_\rho$ of the constituent kernel as in \autoref{eq:conv-spectrum}, for $\rho, \sigma \, {\neq} \, 1$,

\begin{equation}
    \left\langle \Phi_\rho, \Phi_\sigma\right\rangle = \frac{s}{d} \sum_{i,j\in\mathcal{P}} \int p(d^dx) \phi_\rho(\bm{x}_i) \overline{\phi_\sigma(\bm{x}_j)}.
\end{equation}

Separating the term in the sum over patches in which $i$ and $j$ coincide from the others, and since the patches are not overlapping, the RHS can be written as

\begin{equation}
     \frac{s}{d} \sum_{i\in\mathcal{P}} \int p(d^sx_i) \phi_\rho(\bm{x}_i) \overline{\phi_\sigma(\bm{x}_i)} + \sum_{i,j\neq i\in\mathcal{P}} \int p(d^sx_i) \phi_\rho(\bm{x}_i) \int p(d^sx_j) \overline{\phi_\sigma(\bm{x}_j)}.
\end{equation}

From the orthonormality of the eigenfunctions $\phi_\rho$, the first integral is non-zero and equal to one only when $\rho\eq\sigma$, while, from assumption \textit{i)}, $\int p^{(s)}(d^s x) \phi_\rho(\bm{x})\eq0$ for all $\rho\,{>}\,1$, so that the second integral is always zero. Therefore,

\begin{equation}
    \left\langle \Phi_\rho, \Phi_\sigma\right\rangle = \delta_{\rho,\sigma},\text{ for }\rho,\sigma > 1.
\end{equation}

When $\rho\eq1$ and $\sigma\,{\neq}\,1$, $\int p(d^dx) \Phi_1(\bm{x}) \overline{\Phi_\sigma(\bm{x})}\eq0$ from assumption \textit{i)}, i.e. $\Phi_1\eq1$ and $\int p^{(s)}(d^s x) \phi_\rho(\bm{x})\eq0$ for all $\rho\,{>}\,1$. Finally, if $\rho\eq\sigma\eq1$, $\int p(d^dx) \Phi_1(\bm{x}) \overline{\Phi_1(\bm{x})}\eq1$ trivially.

Then, we prove that the eigenfunctions and the eigenvalues defined in \autoref{eq:conv-spectrum} satisfy the kernel eigenproblem. For $\rho\eq1$,

\begin{equation}
    \int p(d^dy) K^{CN}(\bm{x},\bm{y}) = \int p(d^dy) \frac{s^2}{d^2} \sum_{i,j\in\mathcal{P}} C(\bm{x}_i, \bm{y}_j) = \frac{s^2}{d^2} \sum_{i,j\in\mathcal{P}} \lambda_1 = \Lambda_1,
\end{equation}

where we used $\int p^{(s)}(d^sy) C(\bm{x},\bm{y}) \eq \lambda_1$ from assumption \textit{i)}. For $\rho>1$,

\begin{equation}
    \int p(d^dy) K^{CN}(\bm{x},\bm{y}) \Phi_\rho(\bm{y}) = \int p(d^dy) \frac{s^2}{d^2} \sum_{i,j\in\mathcal{P}} C(\bm{x}_i, \bm{y}_j) \sqrt{\frac{s}{d}} \sum_{l\in\mathcal{P}} \phi_\rho(\bm{y}_l).
\end{equation}

Splitting the sum over $l$ into the term with $l \eq j$ and the remaining ones, the RHS can be written as

\begin{align}
    \frac{s^2}{d^2} \sum_{i,j\in\mathcal{P}} &\Biggl( \int p(d^sy_j) C(\bm{x}_i, \bm{y}_j) \sqrt{\frac{s}{d}} \phi_\rho(\bm{y}_j) \\ &+ \int p(d^sy_j) C(\bm{x}_i, \bm{y}_j) \sqrt{\frac{s}{d}} \sum_{l\neq j\in\mathcal{P}} \int p(d^sy_l) \phi_\rho(\bm{y}_l) \Biggr). \nonumber
\end{align}

Using assumption \textit{i)}, the third integral is always zero, therefore

\begin{equation}
    \int p(d^dy) K^{CN}(\bm{x},\bm{y}) \Phi_\rho(\bm{y}) = \frac{s^2}{d^2} \sum_{i,j\in\mathcal{P}} \lambda_\rho \sqrt{\frac{s}{d}} \phi_\rho(\bm{x}_i) = \Lambda_\rho \Phi_\rho(\bm{x}).
\end{equation}

Finally, we prove the expansion of \autoref{eq:conv-decomp} from the definition of $K^{CN}$,

\begin{align}
    K^{CN}(\bm{x}, \bm{y}) &= \frac{s^2}{d^2} \sum_{i,j\in\mathcal{P}} C(\bm{x}_i, \bm{y}_j) \\
    &= \frac{s^2}{d^2} \sum_{i,j\in\mathcal{P}} \sum_\rho \lambda_\rho \phi_\rho(\bm{x}_i) \overline{\phi_\rho(\bm{y}_j)} \nonumber \\
    &= \lambda_1 \frac{s^2}{d^2} \sum_{i,j\in\mathcal{P}} \phi_1(\bm{x}_i) \overline{\phi_1(\bm{y}_j)} + \sum_{\rho>1} \left( \frac{s}{d} \lambda_\rho \right) \left( \sqrt{\frac{s}{d}} \sum_{i\in\mathcal{P}} \phi_\rho(\bm{x}_i) \right) \left( \sqrt{\frac{s}{d}}\sum_{j\in\mathcal{P}} \overline{\phi_\rho(\bm{y}_j)} \right) \nonumber \\
    &= \sum_{\rho} \Lambda_\rho \Phi_\rho(\bm{x}) \overline{\Phi_\rho(\bm{y})}. \nonumber
\end{align}

\end{proof}

\paragraph{Proof of Lemma \ref{lemma:loc-spectra}}

\begin{proof}

We start again by proving the orthonormality of the eigenfunctions. By writing the $d$-dimensional eigenfunctions $\Phi_{\rho,i}$ in terms of the $s$-dimensional eigenfunctions $\phi_\rho$ of the constituent kernel as in \autoref{eq:loc-spectrum}, for $\rho, \sigma \, {\neq} \, 1$,

\begin{equation}
    \left\langle \Phi_{\rho,i}, \Phi_{\sigma,j}\right\rangle = \int p(d^dx) \phi_\rho(\bm{x}_i) \overline{\phi_\sigma(\bm{x}_j)} = \delta_{\rho,\sigma} \delta_{i,j},
\end{equation}

from the orthonormality of the eigenfunctions $\phi_\rho$ when $i \eq j$, and assumption \textit{i)}, $\int p^{(s)}(d^s x) \phi_\rho(\bm{x})\eq0$ for all $\rho\,{>}\,1$, when $i\,{\neq}j$. Moreover, as $\Phi_1(\bm{x})\,{=}\,1$, $\int p(d^dx) \Phi_{1}(\bm{x}) \overline{\Phi_{\sigma\neq1,j}(\bm{x})} \eq 0$ and $\int p(d^dx) \Phi_1(\bm{x}) \overline{\Phi_1(\bm{x})}\eq1$. 

Then, we prove that the eigenfunctions and the eigenvalues defined in \autoref{eq:loc-spectrum} satisfy the kernel eigenproblem. For $\rho\eq1$,

\begin{equation}
    \int p(d^dy) K^{LC}(\bm{x},\bm{y}) = \int p(d^dy) \frac{s}{d} \sum_{i\in\mathcal{P}} C(\bm{x}_i, \bm{y}_i) = \frac{s}{d} \sum_{i\in\mathcal{P}} \lambda_1 = \Lambda_1,
\end{equation}

where we used $\int p^{(s)}(d^sy) C(\bm{x},\bm{y}) \eq \lambda_1$ from assumption \textit{i)}. For $\rho>1$,

\begin{equation}
    \int p(d^dy) K^{LC}(\bm{x},\bm{y}) \Phi_{\rho,i}(\bm{y}) = \int p(d^dy) \frac{s}{d} \sum_{j\in\mathcal{P}} C(\bm{x}_j, \bm{y}_j) \phi_\rho(\bm{y}_i).
\end{equation}

Splitting the sum over $j$ in the term for which $j\eq i$ and the remaining ones, the RHS can be written as

\begin{equation}
    \frac{s}{d} \int p(d^sy_i) C(\bm{x}_i, \bm{y}_i) \phi_\rho(\bm{y}_i) +  \frac{s}{d} \sum_{j \neq i\in\mathcal{P}} \int p(d^sy_j) C(\bm{x}_j, \bm{y}_j) \int p(d^sy_i) \phi_\rho(\bm{y}_i).
\end{equation}

Using assumption \textit{i)}, the third integral is always zero, therefore

\begin{equation}
    \int p(d^dy) K^{CN}(\bm{x},\bm{y}) \Phi_\rho(\bm{y}) = \frac{s}{d} \lambda_\rho \phi_\rho(\bm{x}_i) = \Lambda_{\rho,i} \Phi_{\rho,i}(\bm{x}).
\end{equation}

Finally, we prove the expansion of \autoref{eq:conv-decomp} from the definition of $K^{LC}$,

\begin{align}
    K^{LC}(\bm{x}, \bm{y}) &= \frac{s}{d} \sum_{i\in\mathcal{P}} C(\bm{x}_i, \bm{y}_i) \\
    &= \frac{s^2}{d^2} \sum_{i\in\mathcal{P}} \sum_\rho \lambda_\rho \phi_\rho(\bm{x}_i) \overline{\phi_\rho(\bm{y}_i)} \\
    &= \lambda_1 \frac{s}{d} \sum_{i\in\mathcal{P}} \phi_1(\bm{x}_i) \overline{\phi_1(\bm{y}_i)} + \sum_{\rho>1} \sum_{i\in\mathcal{P}}  \left( \frac{s}{d} \lambda_\rho \right)  \phi_\rho(\bm{x}_i) \overline{\phi_\rho(\bm{y}_i)}  \\
    &= \Lambda_1 \Phi_1(\bm{x}) \overline{\Phi_1(\bm{y})} +  \sum_{\rho>1} \sum_{i\in\mathcal{P}} \Lambda_{\rho,i} \Phi_{\rho,i}(\bm{x}) \overline{\Phi_{\rho,i}(\bm{y})}.
\end{align}

\end{proof}

\subsection{Spectra of convolutional kernels with overlapping patches}\label{asec:mercer-overlapping}

In this section Lemma~\ref{lemma:conv-spectra} and Lemma~\ref{lemma:loc-spectra} are extended to kernels with overlapping patches, having $\mathcal{P}\,{=}\,\left\lbrace 1,\dots, d\right\rbrace$ and $|\mathcal{P}|\,{=}\,d$. 
Such extension requires additional assumptions, which are stated below:
\begin{itemize}
    \item[$i)$] The $d$-dimensional input measure $p^{(d)}(d^d x)$ is uniform on the $d$-torus $[0,1]^d$;
    \item[$ii)$] The constituent kernel $C(\bm{x},\bm{y})$ is translationally-invariant, isotropic and periodic,

    \begin{equation}
        C(\bm{x},\bm{y}) = \mathcal{C}(||\bm{x}-\bm{y}||),\quad \mathcal{C}(||\bm{x}-\bm{y} + \bm{n}||) = \mathcal{C}(||\bm{x}-\bm{y}||) \quad \forall \bm{n}\in\mathbb{Z}^s.
    \end{equation}
\end{itemize}
Assumptions $i)$ and $ii)$ above imply that $C(\bm{x},\bm{y})$ can be diagonalised in Fourier space, i.e. (with $\bm{k}$ denoting the $s$-dimensional wavevector)

\begin{equation}\label{eq:constituent-fourier}
    \mathcal{C}(\bm{x}-\bm{y}) = \sum_{\left\lbrace \bm{k}=2\pi\bm{n}| \bm{n}\in\mathbb{Z}^s\right\rbrace} \lambda_{\bm{k}} \phi_{\bm{k}}(\bm{x}) \overline{\phi_{\bm{k}}(\bm{y})} = \sum_{\left\lbrace \bm{k}=2\pi\bm{n}| \bm{n}\in\mathbb{Z}^s\right\rbrace} \lambda_{\bm{k}} e^{i\bm{k}\cdot(\bm{x}-\bm{y})},
\end{equation}

and the eigenvalues $\lambda_{\bm{k}}$ depend only on the modulus of $\bm{k}$, $k\,{=}\,\sqrt{\bm{k}\cdot\bm{k}}$.

Let us introduce the following definitions, after recalling that a $s$-dimensional patch $\bm{x}_i$ of $\bm{x}$ is a contiguous subsequence of $\bm{x}$ starting at $x_i$, i.e.
\begin{equation}
    \bm{x}=(x_1, x_2, \dots, x_d) \Rightarrow \bm{x}_i = (x_i, x_{i+1}, \dots, x_{i+s-1}),
\end{equation}
and that inputs are `wrapped', i.e. we identify $x_{i+ n d}$ with $x_i$ for all $n\in\mathbb{Z}$.

\begin{itemize}
    \item Two patches $\bm{x}_i$ and $\bm{x}_j$ \emph{overlap} if $\bm{x}_i\displaystyle \cap \bm{x}_j\,{\neq}\,\emptyset$. The overlap $\bm{x}_{i\cap j}\equiv \bm{x}_i\displaystyle \cap \bm{x}_j$ is an $o$-dimensional patch of $\bm{x}$, with $o\,{=}\,|\bm{x}_i\displaystyle \cap \bm{x}_j|$;
    \item let $\mathcal{P}$ denote the set of patch indices associated with a given kernel/architecture. We denote with $\mathcal{P}_i$ the set of indices of patches which overlap with $\bm{x}_i$, i.e. $\mathcal{P}_i\,{=}\,\left\lbrace i-s +1,\dots, i, \dots, i + s-1\right\rbrace\,{=}\,\left\lbrace \mathcal{P}_{-,i}, i ,\mathcal{P}_{+,i} \right\rbrace $;
    \item Given two overlapping patches $\bm{x}_i$ and $\bm{x}_j$ with $o$-dimensional overlap, the union $\bm{x}_{i\cup j}\equiv \bm{x}_i \displaystyle\cup \bm{x}_j$ and differences $\bm{x}_{i \smallsetminus j} \equiv \bm{x}_i \,{\smallsetminus}\, \bm{x}_j$ and $\bm{x}_{j\smallsetminus i} \equiv \bm{x}_j \,{\smallsetminus}\, \bm{x}_i$ are all patches of $\bm{x}$, with dimensions $2s\,{-}\,o$, $s\,{-}\,o$ and $s\,{-}\,o$, respectively.
\end{itemize}

We also use the following notation for denoting subspaces of the $\bm{k}$-space $\cong \mathbb{Z}^s$,

\begin{equation}
    \mathcal{F}^{u} = \left\lbrace \bm{k}\,{=}\,2\pi \bm{n} \, | \, \bm{n}\in\mathbb{Z}^{s}; \, n_1, n_u \neq 0; \, n_{v} = 0 \, \forall v \text{ s. t. } u<v\leq s  \right\rbrace.
\end{equation}

$\mathcal{F}^{s}$ is the set of all wavevectors $\bm{k}$ having nonvanishing extremal components $k_1$ and $k_s$. For $u\,{<}\,s$, $\mathcal{F}^{u}$ is formed by first considering only wavevectors having the last $s-u$ components equal to zero, then asking the resulting $u$-dimensional wavevectors to have nonvanishing extremal components. Practically, $\mathcal{F}^{u}$ contains wavevectors which can be entirely specified by the first $u$-dimensional patch $\bm{k}^{(u)}_1\,{=}\,(k_1,\dots,k_u)$ but not by the first $(u\,{-}\,1)$-dimensional one. Notice that, in order to safely compare $\bm{k}$'s in different $\mathcal{F}$'s, we have introduced an apex $u$ denoting the dimensionality of the patch.

\begin{lemma}[Spectra of overlapping convolutional kernels]\label{lemma:conv-spectra-overlap}
Let $K^{CN}$ be a convolutional kernel defined as in~\autoref{eq:conv-ker}, with $\mathcal{P}\,{=}\,\left\lbrace1,\dots, d\right\rbrace$ and constituent kernel $C$ satisfying assumptions $i)$, $ii)$ above. Then, $K^{CN}$ admits the following Mercer's decomposition,

\begin{equation}\label{eq:conv-decomp-overlap}
   K^{CN}(\bm{x},\bm{y}) = \Lambda_{\bm{0}} + \displaystyle\sum_{u=1}^{s}\left( \displaystyle\sum_{ \bm{k}\in \mathcal{F}^{u}} \Lambda_{\bm{k}} \Phi_{\bm{k}}(\bm{x}) \Phi_{\bm{k}}(\bm{y})\right),
\end{equation}

with eigenfunctions

\begin{equation}\label{eq:conv-eigvec-overlap}
 \Phi_{\bm{0}}(\bm{x})\,{=}\,1, \quad \Phi_{\bm{k}}(\bm{x})\,{=}\,\frac{1}{\sqrt{d}}\sum_{i=1}^d \phi_{\bm{k}}(\bm{x}_i) \quad \forall\, \bm{k}\neq\bm{0},
\end{equation}

and eigenvalues

\begin{equation}\label{eq:conv-eigval-overlap}
 \Lambda_{\bm{0}}\,{=}\,\lambda_{\bm{0}}, \quad \Lambda_{\bm{k}}\,{=}\, \frac{s-u+1}{d}\lambda_{\bm{k}} \quad \forall\, \bm{k}\in\mathcal{F}^{u} \text{ with } u\leq s.
\end{equation}
\end{lemma}

\begin{proof} We start by proving the orthonormality of the eigenfunctions. In general, by orthonormality of the $s$-dimensional plane waves $\phi_{\bm{k}}(\bm{x})$, we have

\begin{align}\label{eq:scalar-product-overlapping}
&\left\langle \Phi_{\bm{k}}, \Phi_{\bm{q}}\right\rangle = \frac{1}{d}\int_{[0,1]^d} d^d x\, \left(\displaystyle\sum_{i=1}^d \phi_{\bm{k}}(\bm{x}_i)\right) \overline{\left(\displaystyle\sum_{j=1}^d \phi_{\bm{q}}(\bm{x}_j)\right)} \nonumber \\
&=\frac{1}{d}\displaystyle\sum_{i\in\mathcal{P}}\displaystyle\sum_{j\notin\mathcal{P}_i} \int d^s x_i\, e^{i\bm{k}\cdot\bm{x}_i} \int d^s x_j \, e^{-i\bm{q}\cdot\bm{x}_j}  + \frac{1}{d}\sum_{i\in\mathcal{P}} \int d^s x_i\, e^{i(\bm{k}-\bm{q})\cdot\bm{x}_i} \nonumber \\
&+ \frac{1}{d}\displaystyle\sum_{i\in\mathcal{P}} \displaystyle\sum_{j\in\mathcal{P}_{i,+}} \int \left(d^{s\text{-}o} x_{i\smallsetminus j}\right) e^{i\bm{k}^{(s-o)}_1\cdot\bm{x}_{i\smallsetminus j}} \int \left(d^{o} x_{i\cup j}\right) e^{i(\bm{k}^{(o)}_{s-o+1}-\bm{q}^{(o)}_1)\cdot\bm{x}_{i\cup j}} \int \left(d^{s\text{-}o} x_{j\smallsetminus i}\right) e^{i\bm{q}^{(s-o)}_{o+1}\cdot\bm{x}_{j\smallsetminus i}} \nonumber \\
& + \frac{1}{d}\displaystyle\sum_{i\in\mathcal{P}} \displaystyle\sum_{j\in\mathcal{P}_{i,-}} \left\lbrace  i\leftrightarrow j, \bm{k} \leftrightarrow \bm{q} \right\rbrace \nonumber \\
&=\frac{1}{d}\displaystyle\sum_{i\in\mathcal{P}} \delta(\bm{k},\bm{0}) \displaystyle\sum_{j\notin\mathcal{P}_i} \delta(\bm{q},\bm{0}) + \frac{1}{d}\displaystyle\sum_{i\in\mathcal{P}} \delta(\bm{k},\bm{q})\nonumber \\
&+\frac{1}{d}\displaystyle\sum_{i\in\mathcal{P}} \left( \displaystyle\sum_{j\in\mathcal{P}_{i,+}} \delta(\bm{k}^{(s-o)}_1,\bm{0}) \, \delta(\bm{k}^{(o)}_{s-o+1}, \bm{q}^{(o)}_1) \, \delta(\bm{q}^{(s-o)}_{o+1}, \bm{0}) \right. \nonumber \\ & \left. + \displaystyle\sum_{j\in\mathcal{P}_{i,-}} \delta(\bm{q}^{(s-o)}_1,\bm{0}) \, \delta(\bm{k}^{(o)}_1, \bm{q}^{(o)}_{s-o+1}) \, \delta(\bm{k}^{(s-o)}_{o+1}, \bm{0})\right),
\end{align}

with $\delta(\bm{k},\bm{q})$ denoting the multidimensional Kronecker delta. For fixed $i$, the three terms on the RHS correspond to $j$'s such that $\bm{x}_j$ does not overlap with $\bm{x}_i$, to $j\,{=}\,i$ and to $j$'s such that $\bm{x}_j$ overlaps with $\bm{x}_i$, respectively. We recall that, in patch notation, $\bm{k}_1^{(s-o)}$ denotes the subsequence of $\bm{k}$ formed with the first $s-o$ components and $\bm{k}_{s-o+1}^{(o)}$ the subsequence formed with the last $o$ components.

By taking $\bm{k}$ and $\bm{q}$ in $\mathcal{F}^s$, as $k_1, k_s\neq 0$ and $q_1,q_s\neq 0$, \autoref{eq:scalar-product-overlapping} implies

\begin{equation}
    \left\langle \Phi_{\bm{k}},\Phi_{\bm{q}}\right\rangle  = \delta(\bm{k},\bm{q}).
\end{equation}

In addition, by taking $\bm{k}\in\mathcal{F}^s$ and $\bm{q}=\bm{q}^{(u)}_1\in\mathcal{F}^u$ with $u\,{<}\,s$,

\begin{equation}
    \left\langle \Phi_{\bm{k}},\Phi_{\bm{q}^{(u)}_1} \right\rangle  = 0\quad \forall\, u<s.
\end{equation}

Thus the $\Phi_{\bm{k}}$'s with $\bm{k}\in\mathcal{F}^s$ are orthonormal between each other and orthogonal to all $\Phi_{\bm{q}^{(u)}_1}$'s with $u\,{<}\,s$. Similarly, by taking $\bm{k}\in\mathcal{F}^u$ with $u\,{<}\,s$ and $\bm{q}\in\mathcal{F}^v$ with $v\,{\leq}\,u$, orthonormality is proven down to $\Phi_{\bm{k}^{(1)}_1}$. The zero-th eigenfunction $\Phi_{\bm{0}}(\bm{x})\,{=}\,1$ is also orthogonal to all other eigenfunctions by~\autoref{eq:scalar-product-overlapping} with $\bm{k}\,{=}\,0$ and trivially normalised to $1$.

Secondly, we prove that eigenfunctions from~\autoref{eq:conv-eigvec-overlap} and eigenvalues from~\autoref{eq:conv-eigval-overlap} satisfy the kernel eigenproblem of $K^{CN}$. For $\bm{k}\,{=}\,\bm{0}$,

\begin{equation}
    \int_{[0,1]^d} d^dy\, K^{CN}(\bm{x},\bm{y}) = \frac{1}{d^2} \sum_{i,j=1}^d \int_{[0,1]^d} d^dy\, \sum_{\bm{q}}\lambda_{\bm{k}} e^{i\bm{q}\cdot(\bm{x}_i-\bm{y}_j)} = \lambda_{\bm{0}},
\end{equation}

proving that $\Lambda_{\bm{0}}$ and $\Phi_{\bm{0}}$ satisfy the eigenproblem. For $\bm{k}\neq\bm{0}$,

\begin{equation}\begin{aligned}
    &\int_{[0,1]^d} d^dy\, K^{CN}(\bm{x},\bm{y})\left(\frac{1}{\sqrt{d}} \sum_{l=1}^d e^{i\bm{k}\cdot\bm{y}_l}\right) = \frac{1}{d^{5/2}} \sum_{i,j,l=1}^d \int_{[0,1]^d} d^dy\, \sum_{\bm{q}}\lambda_{\bm{q}} e^{i\bm{q}\cdot(\bm{x}_i-\bm{y}_j)}e^{i\bm{k}\cdot\bm{y}_l} \\
    &= \frac{1}{d^{5/2}} \sum_{i=1}^d \sum_{\bm{q}} \lambda_{\bm{q}}e^{i\bm{q}\cdot\bm{x}_i}\sum_{j=1}^d \left( \delta(\bm{k},\bm{q}) +  \displaystyle\sum_{l\in\mathcal{P}_{j,+}} \delta(\bm{q}^{(s-o)}_1,\bm{0}) \, \delta(\bm{q}^{(o)}_{s-o+1}, \bm{k}^{(o)}_1) \, \delta(\bm{k}^{(s-o)}_{o+1}, \bm{0}) \right. \\ &  \left.+  \displaystyle\sum_{l\in\mathcal{P}_{j,-}} \delta(\bm{k}^{(s-o)}_1,\bm{0}) \, \delta(\bm{q}^{(o)}_1, \bm{k}^{(o)}_{s-o+1}) \, \delta(\bm{q}^{(s-o)}_{o+1}, \bm{0}) \right).
\end{aligned}\end{equation}

When $\bm{k}\in\mathcal{F}^s$, the deltas coming from the terms with $j\in\mathcal{P}_{j,\pm}$ vanish, showing that the eigenproblem is satisfied with $\Lambda_{\bm{k}}\,{=}\,\lambda_{\bm{k}}/d$ and $\Phi_{\bm{k}}(\bm{x})\,{=}\,\sum_l e^{i\bm{k}\cdot\bm{x}}/\sqrt{d}$. When $\bm{k}\in\mathcal{F}^u$ with $u\,{<}\,s$, as the last $s\,{-}\,u$ components of $\bm{k}$ vanish, there are several $\bm{q}$'s satisfying the deltas in the bracket. There is $\bm{q}\,{=}\,\bm{k}$, from the $l\,{=}\,j$ term, then there are the $s\,{-}\,u$ $\bm{q}$'s such that $\delta(\bm{q}^{(s-o)}_1,\bm{0}) \delta(\bm{q}^{(o)}_{s-o+1}, \bm{k}^{(o)}_1) \delta(\bm{k}^{(s-o)}_{o+1}, \bm{0}) \,{=}\,1$. These are all the $\bm{q}$'s having a $u$-dimensional patch equal to $\bm{k}_1^{(u)}$ and all the other elements set to zero, thus there are $(s-u+1)$ such $\bm{q}$'s. Moreover, as $\lambda_{\bm{q}}$ depends only on the modulus of $\bm{q}$, all these $\bm{q}$'s result in the same eigenvalue, and in the same eigenfunction $\sum_l e^{i\bm{q}\cdot\bm{x}}/\sqrt{d}$, after the sum over patches. Therefore, 

\begin{equation}
    \int_{[0,1]^d} d^dy\, K^{CN}(\bm{x},\bm{y}) \Phi_{\bm{k}_1^{(u)}} = \frac{(s-u + 1)}{d}\lambda_{\bm{k}_1^{(u)}}\Phi_{\bm{k}_1^{(u)}} = \Lambda_{\bm{k}_1^{(u)}}\Phi_{\bm{k}_1^{(u)}}.
\end{equation}

Finally, we prove the expansion of the kernel in~\autoref{eq:conv-decomp-overlap},

\begin{align}\label{eq:decomp-proof}
    K^{CN}(\bm{x}, \bm{y}) &= \frac{1}{d^2} \sum_{i,j\in\mathcal{P}} C(\bm{x}_i, \bm{y}_j) \\
    &= \sum_{\bm{k}} \frac{1}{d}\lambda_{\bm{k}} \left(\frac{1}{\sqrt{d}}\sum_{i\in\mathcal{P}}\phi_{\bm{k}}(\bm{x}_i)\right) \overline{\left(\frac{1}{\sqrt{d}}\sum_{j\in\mathcal{P}}\phi_{\bm{k}}(\bm{y}_j)\right)}.
\end{align}

The terms on the RHS of~\autoref{eq:decomp-proof} are trivially equal to those of~\autoref{eq:conv-decomp-overlap} for $\bm{k}\in\mathcal{F}^s$. All the $\bm{k}$ having $s\,{-}\,u$ vanishing extremal components can be written as shifts of $\bm{k}_1^{(u)}\in\mathcal{F}^{u}$, which has the \emph{last} $s\,{-}\,u$ components vanishing. But a shift of $\bm{k}$ does not affect $\lambda_{\bm{k}}$ nor $\sum_l e^{i\bm{k}\cdot\bm{x}}$, leading to a degeneracy of eigenvalues having $\bm{k}$ which can be obtained from a shift of $\bm{k}_1^{(u)}\in\mathcal{F}^u$. Such degeneracy is removed by restricting the sum over $\bm{k}$ to the sets $\mathcal{F}^u$, $u\,{\leq}\,s$, of wavevectors with non-vanishing extremal components, and rescaling the remaining eigenvalues with a factor of $(s-u+1)/d$, so that~\autoref{eq:conv-decomp-overlap} is obtained.
\end{proof}

\begin{lemma}[Spectra of overlapping local kernels]\label{lemma:loc-spectra-overlap}
Let $K^{LC}$ be a local kernel defined as in~\autoref{eq:loc-ker}, with $\mathcal{P}\,{=}\,\left\lbrace1,\dots, d\right\rbrace$ and constituent kernel $C$ satisfying assumptions $i)$, $ii)$ above. Then, $K^{LC}$ admits the following Mercer's decomposition,

\begin{equation}\label{eq:loc-decomp-overlap}
   K^{LC}(\bm{x},\bm{y}) = \Lambda_{\bm{0}} +  \displaystyle\sum_{u=1}^{s}\left( \displaystyle\sum_{ \bm{k}\in \mathcal{F}^{u}}\displaystyle\sum_{i=1}^d \Lambda_{\bm{k},i} \Phi_{\bm{k},i}(\bm{x}) \Phi_{\bm{k},i}(\bm{y})\right)
\end{equation}

with eigenfunctions

\begin{equation}\label{eq:loc-eigvec-overlap}
 \Phi_{\bm{0}}(\bm{x})\,{=}\,1, \quad \Phi_{\bm{k},i}(\bm{x})\,{=}\,\phi_{\bm{k}}(\bm{x}_i) \quad \forall\, \bm{k}\in\mathcal{F}^u \text{ with }1\leq u \leq s \text{ and }i=1,\dots,d,
\end{equation}

and eigenvalues

\begin{equation}\label{eq:loc-eigval-overlap}
 \Lambda_{\bm{0}}\,{=}\,\lambda_{\bm{0}}, \Lambda_{\bm{k},i} = \frac{s-u+1}{d}\lambda_{\bm{k}} \quad \forall\, \bm{k}\in\mathcal{F}^{u} \text{ with } u\leq s\text{ and }i=1,\dots,d.
\end{equation}
\end{lemma}
\begin{proof} We start by proving the orthonormality of the eigenfunctions. The scalar product $\left\langle\Phi_{\bm{k},i}, \Phi_{\bm{q},j} \right\rangle$ depends on the relation between the $i$-th and $j$-th patches.

\begin{subequations}\label{eq:scalar-product-overlapping-loc}
\begin{align}
\int_{[0,1]^d} d^d x\,  \phi_{\bm{k}}(\bm{x}_i) \overline{ \phi_{\bm{q}}(\bm{x}_j)} & &\nonumber \\
\label{eq:scalar-product-overlapping-loc-pplus} &=\delta(\bm{k}^{(s-o)}_1,\bm{0}) \, \delta(\bm{k}^{(o)}_{s-o+1}, \bm{q}^{(o)}_1) \, \delta(\bm{q}^{(s-o)}_{o+1}, \bm{0}), &\text{ if } j\in\mathcal{P}_{i,+},\\ \label{eq:scalar-product-overlapping-loc-pminus}
& =\delta(\bm{q}^{(s-o)}_1,\bm{0}) \, \delta(\bm{k}^{(o)}_1, \bm{q}^{(o)}_{s-o+1}) \, \delta(\bm{k}^{(s-o)}_{o+1}, \bm{0}), & \text{ if } j\in\mathcal{P}_{i,-},\\
\label{eq:scalar-product-overlapping-loc-zero}& =\delta(\bm{k},\bm{0})\,\delta(\bm{q},\bm{0}), & \text{ if } j\notin\mathcal{P}_{i}, \\ \label{eq:scalar-product-overlapping-loc-equal}
& = \delta(\bm{k},\bm{q}), &\text{ if } j=i.
\end{align}
\end{subequations}

Clearly, $\left\langle \Phi_{\bm{0}}, \Phi_{\bm{0}}\right\rangle\,{=}\,1$ and
setting one of $\bm{q}$ and $\bm{k}$ to $\bm{0}$ in~\autoref{eq:scalar-product-overlapping-loc} yields orthogonality between $\Phi_{\bm{0}}$ and $\Phi_{\bm{k}, i}$ for all $\bm{k}\neq \bm{0}$ and $i\,{=}\,1,\dots,d$. For any $\bm{k}$ and $\bm{q}\neq 0$, \autoref{eq:scalar-product-overlapping-loc-equal} implies

\begin{equation}
 \left\langle\Phi_{\bm{k},i}, \Phi_{\bm{q},j} \right\rangle = \delta(\bm{k},\bm{q})\delta_{i,j}
\end{equation}

unless $\bm{k}\,{=}\,\bm{k}^{(u)}_1\in\mathcal{F}^u$ and $\bm{q}$ is a shift of $\bm{k}^{(u)}$. But such a $\bm{q}$ would have $q_1\,{=}\,0$ and there is no eigenfunction $\Phi_{\bm{q}}$ with $q_1\,{=}\,0$, apart from $\Phi_{\bm{0}}$. Hence, orthonormality is proven.

We then prove that eigenfunctions and eigenvalues defined in~\autoref{eq:loc-eigvec-overlap} and~\autoref{eq:loc-eigval-overlap} satisfy the kernel eigenproblem of $K^{LC}$. For $\bm{k}\,{=}\,\bm{0}$,

\begin{equation}
    \int_{[0,1]^d} d^dy\, K^{LC}(\bm{x},\bm{y}) = \frac{1}{d} \sum_{i=1}^d \int_{[0,1]^d} d^dy\, \sum_{\bm{q}}\lambda_{\bm{k}} e^{i\bm{q}\cdot(\bm{x}_i-\bm{y}_i)} = \lambda_{\bm{0}}.
\end{equation}

In general,

\begin{equation}\begin{aligned}
    &\int_{[0,1]^d} d^dy\, K^{LC}(\bm{x},\bm{y})e^{i\bm{k}\cdot\bm{y}_l} = \frac{1}{d} \sum_{i=1}^d \int_{[0,1]^d} d^dy\, \sum_{\bm{q}}\lambda_{\bm{q}} e^{i\bm{q}\cdot(\bm{x}_i-\bm{y}_i)} e^{i\bm{k}\cdot\bm{y_l}} \\
    &= \frac{1}{d}\sum_{\bm{q}} \lambda_{\bm{q}} \left( \delta(\bm{k},\bm{q}) e^{i\bm{k}\cdot\bm{x}_l} + \sum_{i\notin\mathcal{P}_l}\delta(\bm{q},0) \, \delta(\bm{k},0) \right.\\ &+ \left. \sum_{i\in\mathcal{P}_{l,+}} e^{i\bm{q}\cdot\bm{x}_i}\delta(\bm{k}^{(s-o)}_1,\bm{0}) \, \delta(\bm{k}^{(o)}_{s-o+1}, \bm{q}^{(o)}_1) \, \delta(\bm{q}^{(s-o)}_{o+1}, \bm{0}) \right. \\  &+ \left.\sum_{i\in\mathcal{P}_{l,-}} e^{i\bm{q}\cdot\bm{x}_i}\delta(\bm{q}^{(s-o)}_1,\bm{0}) \, \delta(\bm{k}^{(o)}_1, \bm{q}^{(o)}_{s-o+1}) \, \delta(\bm{k}^{(s-o)}_{o+1}, \bm{0}) \right).
\end{aligned}\end{equation}

For $\bm{k}\,\in\,\mathcal{F}^u$, with $u=1,\dots,s$, the deltas which set the first component of $\bm{k}$ to $0$ are never satisfied, therefore

\begin{equation}\label{eq:eigval-loc-comp}\begin{aligned}
    &\int_{[0,1]^d} d^dy\, K^{LC}(\bm{x},\bm{y})e^{i\bm{k}\cdot\bm{y}_l} \\
    &= \frac{1}{d}\sum_{\bm{q}} \lambda_{\bm{q}} \left( \delta(\bm{k},\bm{q}) e^{i\bm{k}\cdot\bm{x}_l} + \sum_{i\in\mathcal{P}_{l,-}} e^{i\bm{q}\cdot\bm{x}_i}\delta(\bm{q}^{(s-o)}_1,\bm{0}) \, \delta(\bm{k}^{(o)}_1, \bm{q}^{(o)}_{s-o+1}) \, \delta(\bm{k}^{(s-o)}_{o+1}, \bm{0}) \right).
\end{aligned}\end{equation}

The second term in brackets vanishes for $\bm{k}\,\in\,\mathcal{F}^s$ and the eigenvalue equation is satisfied with $\lambda_{\bm{k}, l} =\lambda_{\bm{k}}/d$. For $\bm{k}=\bm{k}^{(u)}_1\,\in\,\mathcal{F}^u$ with $u\,{<}\,s$, $\delta(\bm{k}^{(s-o)}_{o+1}, \bm{0})\,{=}\,1$ for any $o\,{\geq}\,u$. As a result of the remaining deltas, the RHS of~\autoref{eq:eigval-loc-comp} becomes a sum over all $\bm{q}$'s which can be obtained from shifts of $\bm{k}^{(u)}_1$, which are $s\,{-}\,u\,{+}\,1$ (including $\bm{k}^{(u)}_1$ itself). The patch $\bm{x}_i$ which is multiplied by $\bm{q}$ in the exponent is also a shift of $\bm{x}_l$, thus all the factors $e^{i\bm{q}\cdot\bm{x}_i}$ appearing in the sum coincide with $e^{i\bm{k}^{(u)}_1\cdot\bm{x}_i}$. As $\lambda_{\bm{q}}$ depends on the modulus of $\bm{q}$, all these terms correspond to the same eigenvalue, $\lambda_{\bm{k}^{(u)}_1}$, so that 

\begin{equation}
    \int_{[0,1]^d} d^dy\, K^{LC}(\bm{x},\bm{y})e^{i\bm{k}^{(u)}_1\cdot\bm{y}_l} = \left(\frac{s-u+1}{d} \lambda_{\bm{k}^{(u)}_1}\right) e^{i\bm{k}^{(u)}_1\cdot\bm{x}_l}.
\end{equation}

Finally, we prove the expansion of the kernel in~\autoref{eq:loc-decomp-overlap},

\begin{align}\label{eq:loc-decomp-proof}
    K^{LC}(\bm{x}, \bm{y}) &= \frac{1}{d} \sum_{i\in\mathcal{P}} C(\bm{x}_i, \bm{y}_i) = \sum_{\bm{k}} \frac{1}{d}\lambda_{\bm{k}}
    \sum_{i\in\mathcal{P}} \phi_{\bm{k}}(\bm{x}_i) \overline{\phi_{\bm{k}}(\bm{y}_i)}.
\end{align}

As in the proof of the eigendecomposition of convolutional kernels, all the $\bm{k}$ having $s\,{-}\,u$ vanishing extremal components can be written as shifts of $\bm{k}_1^{(u)}\in\mathcal{F}^{u}$, which has the \emph{last} $s\,{-}\,u$ components vanishing. The shift of $\bm{k}$ does not affect $\lambda_{\bm{k}}$ nor the product $\phi_{\bm{k}}(\bm{x}_i) \overline{\phi_{\bm{k}}(\bm{y}_i)}$, after summing over $i$ leading to a degeneracy of eigenvalues which is removed by restricting the sum over $\bm{k}$ to the sets $\mathcal{F}^u$, $u\,{\leq}\,s$, and rescaling the remaining eigenvalues $\lambda_{\bm{k}_1^{(u)}}$ with a factor of $(s-u+1)/d$, leading to~\autoref{eq:loc-decomp-overlap}.
\end{proof}

\section{Proof of Theorem \ref{th:scaling}}\label{app:thm1}

\begin{theorem}[Theorem \ref{th:scaling} in the main text]
Let $K_T$ be a $d$-dimensional convolutional kernel with a translationally-invariant $t$-dimensional constituent and leading nonanalyticity at the origin controlled by the exponent $\alpha_t\,{>}\,0$. Let $K_S$ be a $d$-dimensional convolutional or local student kernel with a translationally-invariant $s$-dimensional constituent, and with a nonanalyticity at the origin controlled by the exponent $\alpha_s\,{>}\,0$. 
Assume, in addition, that  if the kernels have overlapping patches then $s\geq t$; whereas if the kernels have nonoverlapping patches $s$ is an integer multiple of $t$; and that data are uniformly distributed on a $d$-dimensional torus. Then, the following asymptotic equivalence holds in the limit $P\to\infty$,
\begin{equation}
    \mathcal{B}(P) \sim P^{-\beta}, \quad \beta = \alpha_t / s.
\end{equation}
\end{theorem}

\begin{proof}

For the sake of clarity, we start with the proof in the nonoverlapping-patches case, and then extend it to the overlapping-patches case. Since $K_T$ and $K_S$ have translationally-invariant constituent kernels and data are uniformly distributed on a $d$-dimensional torus, the kernels can be diagonalised in Fourier space. Let us start by considering a convolutional student: because of the constituent kernel's isotropy, the Fourier coefficients $\Lambda^{(s)}_{\bm{k}}$ of $K_S$ depend on $k$ (modulus of $\bm{k}$) only. Notice the superscript indicating the dimensionality of the student constituents. In particular, $\Lambda^{(s)}_{\bm{k}}$ is a decreasing function of $k$ and, for large $k$, $\Lambda_{\bm{k}}\sim k^{-(s+\alpha_s)}$.  Then, $\mathcal{B}(P)$ reads

\begin{equation}\label{eq:Bp}
 \mathcal{B}(P) = \sum_{\left\lbrace \bm{k}| k > k_c(P) \right\rbrace} \mathbb{E}[|c_{\bm{k}}|^2],
\end{equation}

where $k_c(P)$ is defined as the wavevector modulus of the $P$-th largest eigenvalue and $\mathbb{E}[|c_{\bm{k}}|^2]$ denotes the variance of the target coefficients in the student eigenbasis. $k_c(P)$ is such that there are exactly $P$ eigenvalues with $k\,{\leq}\,k_c(P)$,

\begin{equation}\label{eq:radius}
    P = \sum_{\left\lbrace \bm{k}|k<k_c(P)\right\rbrace} 1 \sim \int \frac{d^s k}{(2\pi)^s} \theta(k_c(P)-k) = \frac{1}{(2\pi)^s} \frac{\pi^{s/2}}{\Gamma(s/2 + 1)} k_c(P)^{s},
\end{equation}
i.e. $k_c(P)\sim P^{1/s}$.

By denoting the eigenfunctions of the student with $\Phi_{\bm{k}}^{(s)}$, the superscript $(s)$ indicating the dimension of the constituent's plane waves,

\begin{align}\label{eq:convconv}
  \mathbb{E}[|c_{\bm{k}}|^2] &= \int_{[0,1]^d} d^dx \, \Phi_{\bm{k}}^{(s)}(\bm{x})\int_{[0,1]^d} d^dy \, \overline{\Phi_{\bm{k}}^{(s)}(\bm{y})} \mathbb{E}[f^*(\bm{x})f^*(\bm{y})]\\ & =\int_{[0,1]^d} d^dx \, \Phi_{\bm{k}}^{(s)}(\bm{x})\int_{[0,1]^d} d^dy \, \overline{\Phi_{\bm{k}}^{(s)}(\bm{y})} K_T(\bm{x}, \bm{y}). \nonumber
\end{align}

Decomposing the teacher kernel $K_T$ into its eigenvalues $\Lambda_{\bm{q}}^{(t)}$ and eigenfunctions $\Phi_{\bm{q}}^{(t)}(\bm{y})$,

\begin{align}
 \mathbb{E}[|c_{\bm{k}}|^2] = \int_{[0,1]^d} d^dx \, \Phi^{(s)}_{\bm{k}}(\bm{x})\int_{[0,1]^d} d^dy \, \overline{\Phi^{(s)}_{\bm{k}}(\bm{y})} &\Biggl(\Lambda^{(t)}_{\bm{0}} \\ &+ \frac{s}{d} \sum_{\bm{q}\neq\bm{0}} \Lambda^{(t)}_{\bm{q}} \sum_{i\in\mathcal{P}^{(t)}} \phi^{(t)}_{\bm{q}}(\bm{x}_i) \sum_{j\in\mathcal{P}^{(t)}} \overline{\phi^{(t)}_{\bm{q}}(\bm{y}_j)}\Biggr). \nonumber
\end{align}

The $\bm{q}\eq\bm{0}$ mode of the teacher can give non-vanishing contributions to $c_{\bm{0}}$ only, therefore it does not enter any term of the sum in \autoref{eq:Bp}. Once we removed the  term with $\bm{q}\eq\bm{0}$, consider the  $\bm{y}$-integral:

\begin{align}\label{eq:Ik}
    \mathcal{I}_{\bm{k}}(\bm{x}) &= \int_{[0,1]^d} d^dy \, \sqrt{\frac{s}{d}} \sum_{m\in\mathcal{P}^{(s)}}  \overline{\phi^{(s)}_{\bm{k}}(\bm{y}_m)} \frac{s}{d} \sum_{\bm{q}\neq\bm{0}} \Lambda^{(t)}_{\bm{q}} \sum_{i\in\mathcal{P}^{(t)}} \phi^{(t)}_{\bm{q}}(\bm{x}_i) \sum_{j\in\mathcal{P}^{(t)}} \overline{\phi^{(t)}_{\bm{q}}(\bm{y}_j)} \\
    &= \left(\frac{s}{d}\right)^{\frac{3}{2}} \sum_{\bm{q}\neq\bm{0}} \Lambda^{(t)}_{\bm{q}} \sum_{i\in\mathcal{P}^{(t)}} \phi^{(t)}_{\bm{q}}(\bm{x}_i) \sum_{m\in\mathcal{P}^{(s)}} \sum_{j\in\mathcal{P}^{(t)}} \int_{[0,1]^d} d^dy\, \overline{\phi^{(s)}_{\bm{k}}(\bm{y}_m)} \,    \overline{\phi^{(t)}_{\bm{q}}(\bm{y}_j)}. \nonumber
\end{align}

As all the $t$-dimensional patches of the teacher must be contained in at least one of the $s$-dimensional patches of the student, in the nonoverlapping case we require that $s$ is an integer multiple of $t$. Then, each of the teacher patches is entirely contained in one and only one patch of the student. If the teacher patch $\bm{y}_j$ is not contained in the student patch $\bm{y}_m$, we can factorise the integration over $\bm{y}$ into two integrals over $\bm{y}_j$ and $\bm{y}_m$. These terms give vanishing contributions to $\mathcal{I}_{\bm{k}}(\bm{x})$ since the integral of a plane wave over a period is always zero for non-zero wavevectors. Instead, if the teacher patch $\bm{y}_j$ is contained in the student patch $\bm{y}_m$, denoting with $l$ the index of the element of $\bm{y}_m$ which coincide with the first element of $\bm{y}_j$, we can factorise the student eigenfunctions as follows

\begin{equation}\label{eq:fact}
 \phi^{(s)}_{\bm{k}}(\bm{y}_m) = \phi^{(t)}_{\bm{k}^{(t)}_l}(\bm{y}_j) \phi^{(s-t)}_{\bm{k}\smallsetminus\bm{k}^{(t)}_l}(\bm{y}_{m\smallsetminus j}).
\end{equation}

Here $\bm{k}^{(t)}_l$ denotes the $t$-dimensional patch of $\bm{k}$ starting at $l$ and $\bm{k}\smallsetminus\bm{k}^{(t)}_l$ the sequence of elements which are in $\bm{k}$ but not in $\bm{k}^{(t)}_l$. As $s$ is an integer multiple of $t$, $l\,{=}\,\tilde{l} \times s/t$ with $\tilde{l}=1,\dots,t$. Inserting \autoref{eq:fact} into \autoref{eq:Ik},

\begin{equation}
 \mathcal{I}_{\bm{k}}(\bm{x}) = \sum_{l=\tilde{l}s/t,\, \tilde{l}=1}^t \delta(\bm{k}\smallsetminus\bm{k}^{(t)}_l,\bm{0}) \, \Lambda_{\bm{k}^{(t)}_l}^{(t)} \sqrt{\frac{s}{d}} \sum_{i\in\mathcal{P}^{(t)}} \overline{\phi_{\bm{k}^{(t)}_l}^{(t)}(\bm{x}_i)}.
\end{equation}

The $\bm{x}$-integral of \autoref{eq:convconv} can be performed by the same means after expanding $\Phi^{(s)}_{\bm{k}}$ as a sum of $s$-dimensional plane waves, so as to get,

\begin{equation}\label{eq:ck2}
 \mathbb{E}[|c_{\bm{k}}|^2] = \sum_{l=\tilde{l}s/t,\, \tilde{l}=1}^t\delta(\bm{k}\smallsetminus\bm{k}^{(t)}_l,\bm{0}) \, \Lambda_{\bm{k}^{(t)}_l}^{(t)}.
\end{equation}

Therefore, $\mathbb{E}[|c_{\bm{k}}|^2]$ is non-zero only for $\bm{k}$'s which have at most $t$ consecutive components greater or equal than zero, and the remaining $s\,{-}\,t$ being strictly zero. Inserting \autoref{eq:ck2} into \autoref{eq:Bp},

\begin{equation}\label{eq:final}
 \mathcal{B}(P) = \sum_{\left\lbrace \bm{k}| k > k_c(P)\right\rbrace} \sum_{l=\tilde{l}s/t,\, \tilde{l}=1}^t \delta(\bm{k}\smallsetminus\bm{k}^{(t)}_l,\bm{0}) \, \Lambda_{\bm{k}^{(t)}_l}^{(t)} \sim \int_{P^{1/s}}^\infty dk k^{t-1} k^{-(\alpha_t + t)} \sim P^{-\frac{\alpha_t}{s}}.
\end{equation}

When using a local student, the convolutional eigenfunctions in the RHS of~\autoref{eq:convconv} are replaced by the local eigenfunctions $\Phi_{\bm{k}, i}(\bm{x})$ of~\autoref{eq:loc-decomp}. 
Repeating the same computations, one finds 
\begin{equation}
    k_c \sim  \left(\frac{P}{d/s}\right)^{\frac{1}{s}},
\end{equation}
\begin{equation}\label{eq:final2}
 \mathbb{E}[|c_{\bm{k},i}|^2] = \frac{s}{d} \sum_{l=\tilde{l}s/t,\, \tilde{l}=1}^t\delta(\bm{k}\smallsetminus\bm{k}^{(t)}_l,\bm{0}) \, \Lambda_{\bm{k}^{(t)}_l}^{(t)}.
\end{equation}
As a result,
\begin{align}
 \mathcal{B}(P) &= \sum_{i \in \mathcal{P}} \sum_{\left\lbrace \bm{k}| k > k_c(P)\right\rbrace} \frac{s}{d} \sum_{l=\tilde{l}s/t,\, \tilde{l}=1}^t \delta(\bm{k}\smallsetminus\bm{k}^{(t)}_l,\bm{0}) \, \Lambda_{\bm{k}^{(t)}_l}^{(t)} \\ 
 &\sim \int_{\left(\frac{P}{d/s}\right)^{\frac{1}{s}}}^\infty dk k^{t-1} k^{-(\alpha_t + t)} \sim \left(\frac{P}{d/s}\right)^{-\frac{\alpha_t}{s}}.
\end{align}


As we showed in \autoref{app:mercer-overlap}, when the patches overlap the set of wavevectors which index the eigenvalues is restricted from $\mathbb{Z}^s$ to the union of the $\mathcal{F}^u$'s for $u\,{=}\,0,\dots,s$. In addition, the eigenvalues with $\bm{k} \in \mathcal{F}^{u}$, $0\,{<}\,u\,{<}\,s$, are rescaled by a factor $(s-u+1)/d$. Therefore, in the overlapping case the eigenvalues do not decrease monotonically with $k$ and $\mathcal{B}(P)$ cannot be written as a sum of over $\bm{k}$'s with modulus $k$ larger than a certain threshold $k_c$. By considering also that, with $t\,{\leq}\,s$, $\mathbb{E}[|c_{\bm{k}}|^2]$ is non-zero only for $\bm{k}$'s which have at most $t$ consecutive nonvanishing components, we have

\begin{equation}\label{eq:Bp-overlap}
 \mathcal{B}(P) = \sum_{u=0}^t \sum_{\bm{k}\in\mathcal{F}^u}\mathbb{E}[|c_{\bm{k}}|^2]\chi(\Lambda^{(s)}_{\bm{k}}\,{>}\,\Lambda_P),
\end{equation}

where $\Lambda_P$ denotes the $P$-th largest eigenvalue and the indicator function $\chi(\Lambda^{(s)}_{\bm{k}}\,{>}\,\Lambda_P)$ ensures that the sum runs over all but the first $P$ eigenvalues of the student. The sets $\{\mathcal{F}^{u}\}_{u<t}$ have all null measure in $\mathbb{Z}^t$, whereas $\mathcal{F}^t$ is dense in $\mathbb{Z}^t$, thus the asymptotics of $\mathcal{B}(P)$ are dictated by the sum over $\mathcal{F}^t$. When $\bm{k}$'s are restricted to the latter set, eigenvalues are again decreasing functions of $k$ and the constraint $\Lambda^{(s)}_{\bm{k}}\,{>}\,\Lambda_P$ translates into $k\,{>}\,k_c(P)$. Having changed, with respect to the nonoverlapping case, only an infinitesimal fraction of the eigenvalues, the asymptotic scaling of $k_c(P)$ with $P$ remains unaltered and the estimates of~\autoref{eq:final} and~\autoref{eq:final2} extend to kernels with nonoverlapping patches after substituting the degeneracy $d/s$ with $|\mathcal{P}|=d$.



\end{proof}

\section{Asymptotic learning curves with a local teacher}\label{app:local}

\begin{theorem} Let $K_T$ be a $d$-dimensional local kernel with a translationally-invariant $t$-dimensional constituent and leading nonanalyticity at the origin controlled by the exponent $\alpha_t\,{>}\,0$. Let $K_S$ be a $d$-dimensional local student kernel with a translationally-invariant $s$-dimensional constituent, and with a nonanalyticity at the origin controlled by the exponent $\alpha_s\,{>}\,0$. 
Assume, in addition, that  if the kernels have overlapping patches then $s\geq t$; whereas if the kernels have nonoverlapping patches $s$ is an integer multiple of $t$; and that data are uniformly distributed on a $d$-dimensional torus. Then, the following asymptotic equivalence holds in the limit $P\to\infty$,
\begin{equation}
    \mathcal{B}(P) \sim P^{-\beta}, \quad \beta = \alpha_t / s.
\end{equation}
\end{theorem}

\begin{proof}

The proof is analogous to that of~\autoref{app:thm1}, the only difference being that eigenfunctions and eigenvalues are indexed by $\bm{k}$ and the patch index $i$. This results in an additional factor of $d/s$ in the RHS of~\autoref{eq:radius}. All the discussion between~\autoref{eq:convconv} and~\autoref{eq:ck2} can be repeated by attaching the additional patch index $i$ to all coefficients. ~\autoref{eq:final} for $\mathcal{B}(P)$ is then corrected with an additional sum over patches. The extra sum, however, does not influence the asymptotic scaling with $P$.

\end{proof}

\section{Proof of Theorem \ref{th:ridge}}\label{app:thm2}

\begin{theorem}[Theorem \ref{th:ridge} in the main text]
Let us consider a positive-definite kernel $K$ with eigenvalues $\Lambda_\rho$, $\sum_\rho \Lambda_\rho < \infty$, and eigenfunctions ${\Phi_\rho}$ learning a (random) target function $f^*$ in kernel ridge regression (\autoref{eq:argmin}) with ridge $\lambda$ from $P$ observations $f^*_{\mu}\,{=}\,f^*(\bm{x}^\mu)$, with $\bm{x}^\mu\in \mathbb{R}^d$ drawn from a certain probability distribution. Let us denote with $\mathcal{D}_T(\Lambda)$ the reduced density of kernel eigenvalues with respect to the target and $\epsilon(\lambda,P)$ the generalisation error and also assume that
\begin{itemize}
    \item[$i)$] For any $P$-tuple of indices $\rho_1,\dots,\rho_P$, the vector $(\Phi_{\rho_1}(\bm{x}^1), \dots,\Phi_{\rho_P}(\bm{x}^P))$ is a Gaussian random vector;
    \item[$ii)$] The target function can be written in the kernel eigenbasis with coefficients $c_\rho$ and $c^2(\Lambda_\rho)\,{=}\,\mathbb{E}[|c_\rho|^2]$, with $\mathcal{D}_T(\Lambda) \sim \Lambda^{-(1+r)}$, $c^2(\Lambda) \sim \Lambda^{q}$ asymptotically for small $\Lambda$ and $r\,{>}\,0$, $r\,{<}\,q\,{<}\,r\,{+}\,2$;
\end{itemize}
Then the following equivalence holds in the joint $P\to\infty$ and $\lambda\to 0$ limit with $1/(\lambda\sqrt{P})\to 0$:
\begin{equation}\label{eq:thm-ridge-result-app}
    \epsilon(\lambda, P) \sim  \sum_{\left\lbrace \rho|\Lambda_\rho < \lambda \right\rbrace} \mathbb{E}{[|c_\rho|^2]} = \int_0^{\lambda} d\Lambda \mathcal{D}_T(\Lambda) c^2(\Lambda).
\end{equation}
\end{theorem}

\begin{proof}

In this proof we make use of results derived in~\cite{jacot2020kernel}. Our setup for kernel ridge regression correspond to what the authors of~\cite{jacot2020kernel} call the \emph{classical setting}. Let us introduce the integral operator $T_K$ associated with the kernel, defined by

\begin{equation}
  (T_K f)(\bm{ x}) = \int p\left( d^d y\right) K(\bm{x},\bm{y}) f(\bm{y}). 
\end{equation}

The trace $Tr[T_K]$ coincide with the sum of $K$'s eigenvalues and is finite by hypothesis. We define the following estimator of the generalisation error $\epsilon(\lambda, P)$,

\begin{equation}
    \mathcal{R}(\lambda, P) = \partial_\lambda \vartheta(\lambda) \int p(d^dx)\, \left( f^*(\bm{x}) - (\mathcal{A}_{\vartheta}f^*)(\bm{x}) \right)^2,
\end{equation}

where $\vartheta(\lambda)$ is the \emph{signal capture threshold} (SCT)~\cite{jacot2020kernel} and $\mathcal{A}_{\vartheta}\,{=}\, T_K (T_K + \vartheta(\lambda))^{-1}$ is a reconstruction operator~\cite{jacot2020kernel}. The target function can be written in the kernel eigenbasis by hypothesis (with coefficients $c_\rho$) and $T_K$ has the same eigenvalues and eigenfunctions of the kernel by definition. Hence,

\begin{equation}
    \mathcal{R}(\lambda, P) = \partial_\lambda \vartheta(\lambda)\sum_{\rho=1}^{\infty} \frac{\vartheta(\lambda)^2}{(\Lambda_\rho + \vartheta(\lambda))^2} |c_\rho|^2 = \partial_\lambda \vartheta(\lambda) \int_0^{\infty}d\Lambda\, \mathcal{D}_T(\Lambda) c^2(\Lambda) \frac{\vartheta(\lambda)^2}{(\Lambda + \vartheta(\lambda))^2},
\end{equation}

where $\mathcal{D}_T$ is the reduced density of eigenvalues~\autoref{eq:reduced-density}. We now derive the asymptotics of $\mathcal{R}(\lambda, P)$ in the joint $P\to\infty$ and $\lambda\to0$ limit, then relate the asymptotics of $\mathcal{R}$ to those of $\epsilon(\lambda,P)$ via a theorem proven in~\cite{jacot2020kernel}.

Proposition 3 of~\cite{jacot2020kernel} shows that for any $\lambda\,{>}\,0$, $\partial_\lambda \vartheta(\lambda)\to 1$ and $\vartheta(\lambda)\to \lambda$ with corrections of order $1/N$. Thus, we focus on the following integral,

\begin{equation}\label{eq:integral1}
    \int_0^{\infty}d\Lambda\, \mathcal{D}_T(\Lambda) c^2(\Lambda) \frac{\lambda^2}{(\Lambda + \lambda)^2}.
\end{equation}

The functions $\mathcal{D}_T(\Lambda)$ and $c^2(\Lambda)$ can be safely replaced with their small-$\Lambda$ expansions $\Lambda^{-(1+r)}$ and $\Lambda^{q}$ over the whole range of the integral above because of the assumptions $q\,{>}\,r$ and $q\,{\leq}\,r+2$. In practice, there should be an upper cut-off on the integral coinciding with the largest eigenvalue $\Lambda_1$, but the assumption $q\,{\leq}\,r+2$ causes this part of the spectrum to be irrelevant for the asymptotics of the error. In fact, we will conclude that the integral is dominated by the portion of the domain around $\lambda$. After the change of variables $y\,{=}\,\Lambda/\lambda$,

\begin{equation}\label{eq:integral2}
    \int_0^{\infty} d\Lambda\,\mathcal{D}_T(\Lambda) c^2(\Lambda) \frac{\lambda^2}{(\Lambda + \lambda)^2} = \lambda^{q-r} \int dy\, \frac{y^{q-1-r}}{(1+y)^2},
\end{equation}

where one recognises one of the integral representations of the beta function,

\begin{equation}
    B(a, b) = \int dy\, \frac{y^{a-1}}{(1+y)^{a+b}} = \frac{\Gamma(a)\Gamma(b)}{\Gamma(a+b)},
\end{equation}

with $\Gamma$ denoting the gamma function. Therefore, 

\begin{equation}\label{eq:integral3}
    \int_0^{\infty} d\Lambda\,\mathcal{D}_T(\Lambda) c^2(\Lambda) \frac{\lambda^2}{(\Lambda + \lambda)^2} = \lambda^{q-r} \frac{\Gamma(q-r)\Gamma(2-q+r)}{\Gamma(2)}.
\end{equation}

It is interesting to notice how the assumptions $q\,{>}\,r$ and $q\,{<}\,r\,{+}\,2$ are required in order to avoid the poles of the $\Gamma$ functions in the RHS of~\autoref{eq:integral3}. 

We now use~\autoref{eq:integral3} to infer the asymptotics of $\mathcal{R}(\lambda, P)$ in the scaling limit $\lambda\to0$ and $P\to\infty$ with $1/(\lambda\sqrt{P})\to0$. The latter condition implies that $\lambda$ decays more slowly than $(P)^{-1/2}$, thus additional terms stemming from the finite-$P$ difference between $\vartheta $ and $\lambda$, of order $P^{-1}$ are negligible w.r.t. $\lambda^{q-r}$. The finite-$P$ difference between $\partial_\lambda\vartheta$, also $O(P^{-1})$, can be neglected too. Finally, 
\begin{equation}\label{eq:integral4}
    \mathcal{R}(\lambda, P) \sim \int_0^{\infty} d\Lambda\,\mathcal{D}_T(\Lambda) c^2(\Lambda) \frac{\lambda^2}{(\Lambda + \lambda)^2}\sim \lambda^{q-r} \sim \int_0^\lambda d\Lambda \mathcal{D}_T(\Lambda) c^2(\Lambda).
\end{equation}

Theorem 6 of~\cite{jacot2020kernel} shows the convergence of $\epsilon(\lambda, P)$ towards $\mathcal{R}(\lambda, P)$ when $P\to\infty$. Specifically,

\begin{equation}
    | \epsilon(\lambda,P) - \mathcal{R}(\lambda, P)| \leq \left(\frac{1}{P} + g\left(\frac{Tr[T_K]}{\lambda\sqrt{P}}\right)\right) \mathcal{R}(\lambda, P),
\end{equation}

where $g$ is a polynomial with non-negative coefficients and $g(0)\,{=}\,0$. With a decaying ridge $\lambda(P)$ such that $1/(\lambda\sqrt{P})\to 0$, both $\mathcal{R}/P$ and $\mathcal{R}g(Tr[T_K]/(\lambda\sqrt{P}))$ tend to zero faster than $\mathcal{R}$ itself, thus the asymptotics of $\epsilon(\lambda,P)$ coincide with those of $\mathcal{R}(\lambda, P)$ and~\autoref{eq:thm-ridge-result-app} is proven.
\end{proof}

\paragraph{Remark} The estimate for the exponent $\beta$ of Corollary~\ref{cor:beta-rigorous} follows from the theorem above with $r\,{=}\,t/(s+\alpha_s)$, $q\,{=}\,(\alpha_t + t)/(\alpha_s + s)$ and $\lambda\sim P^{-\gamma}$. Then $q\,{>}\,r$ because $\alpha_t\,{>}\,0$, whereas the condition $q\,{<}\,r + 2$ is equivalent to the assumption $\alpha_t \,{<}\,2(\alpha_s + s)$ required in~\autoref{sec:learning-curves} in order to derive the learning curve exponent in~\autoref{eq:prediction} from our estimate of $\mathcal{B}(P)$.

\section{Numerical experiments}\label{app:numerics}
 
\subsection{Details on the simulations}

To obtain the empirical learning curves, we generate $P+P_{\text{test}}$ random points uniformly distributed in a $d$-dimensional hypercube or on the surface of a $d-1$-dimensional hypersphere embedded in $d$ dimensions. We use $P \in \{128, 256, 512, 1024, 2048, 4096, 8192\}$ and $P_{\text{test}}=8192$. For each value of $P$, we generate a Gaussian random field with covariance given by the teacher kernel, and we compute the kernel ridgeless regression predictor of the student kernel using \autoref{eq:krrpredictor} with the $P$ training samples. The generalisation error defined in \autoref{eq:test-def} is approximated by computing the empirical mean squared error on the $P_{\text{test}}$ unseen samples. The expectation with respect to the target function is obtained averaging over 128 independent teacher Gaussian processes, each sampled on different points of the domain. As teacher and student kernels, we consider different combinations of the convolutional and local kernels defined in \autoref{eq:conv-ker} and \autoref{eq:loc-ker}, with Laplacian constituents $\mathcal{C}(\bm{x}_i-\bm{x}_j) \, {=} \, e^{-\|\bm{x}_i-\bm{x}_j\|}$ and overlapping patches. In particular,
\begin{itemize}
\item the cases with convolutional teacher and both convolutional and local students with various filter sizes are reported in \autoref{fig:figure} and \autoref{fig:sphere} for data distributed in a hypercube and on a hypersphere respectively;
\item the cases with local teacher and both local and convolutional students are reported in \autoref{fig:loc} for data distributed in a hypercube.
\end{itemize}

Experiments are run on NVIDIA Tesla V100 GPUs using the PyTorch package. The approximate total amount of time to reproduce all experiments with our setup is 400 hours. Code for reproducing the experiments can be found at \url{https://github.com/fran-cagnetta/local_kernels}.

\subsection{Additional experiments}

\begin{figure}
    \centering
    \includegraphics[width=1.0\linewidth]{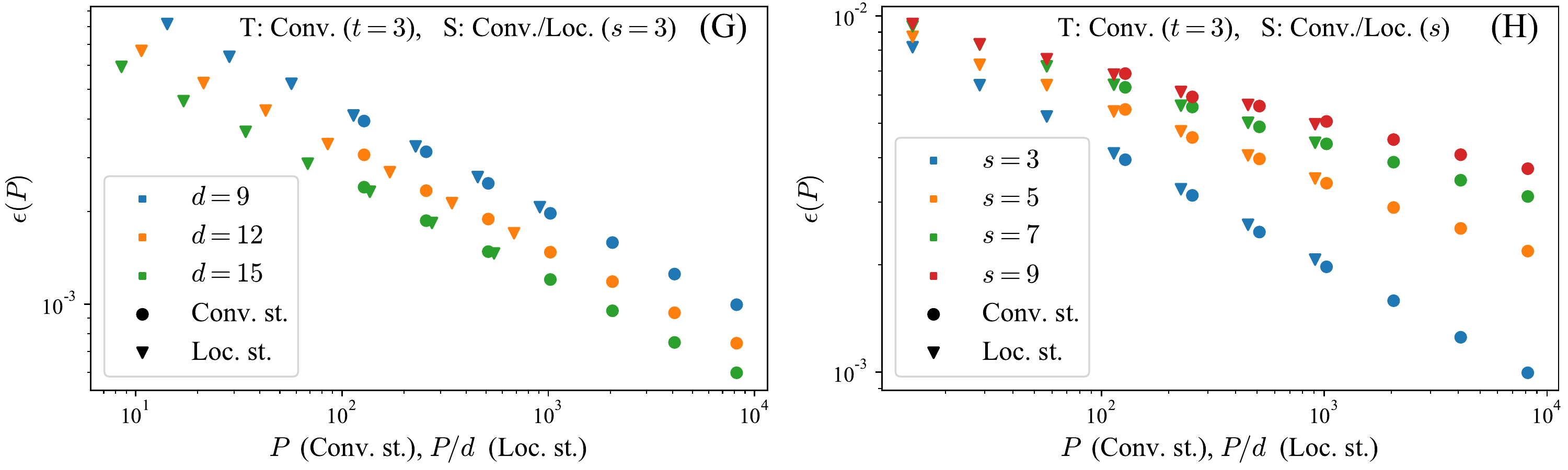}
    \caption{Learning curves for convolutional teacher and local and convolutional student kernels, with filter sizes denoted by $t$ and $s$ respectively. Data are sampled uniformly in the hypercube $[0,1]^d$, with $d=9$ if not specified otherwise. The sample complexity $P$ of the local students is rescaled with the number of patches to highlight the pre-asymptotic effect of shift-invariance on the learning curves.}
    \label{fig:preasympt}
\end{figure}

\paragraph{Convolutional vs local students} In \autoref{fig:preasympt} we report the empirical learning curves for convolutional and local student kernels learning a convolutional teacher kernel, with filter sizes $s$ and $t$ respectively. Data are uniformly sampled in the hypercube $[0,1]^d$. By rescaling the sample complexity $P$ of the local students with the number of patches $|\mathcal{P}|=d$, the learning curves of local and convolutional students overlap, confirming our prediction on the role of shift-invariance. Indeed, the local student has to learn the same local task at all the possible patch locations, while the convolutional student is naturally shift-invariant.

\begin{figure}
    \centering
    \includegraphics[width=1.0\linewidth]{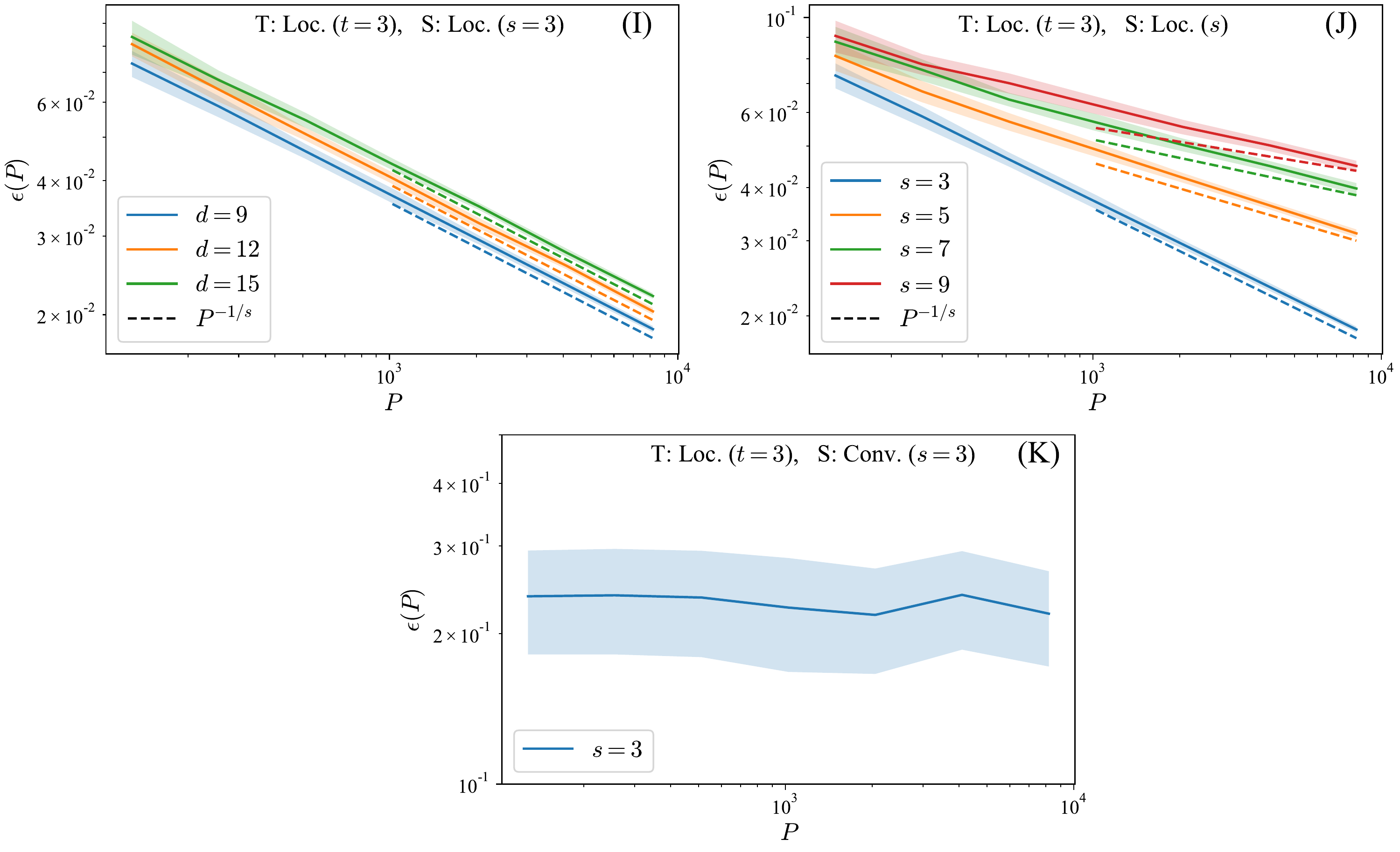}
    \caption{Learning curves for local teacher and local and convolutional student kernels, with filter sizes denoted by $t$ and $s$ respectively. Data are sampled uniformly in the hypercube $[0,1]^d$, with $d=9$ if not specified otherwise. Solid lines are the results of numerical experiments averaged over 128 realisations and the shaded areas represent the empirical standard deviations. The predicted scaling are shown by dashed lines.}
    \label{fig:loc}
\end{figure}

\paragraph{Local teacher} In \autoref{fig:loc} we report the empirical learning curves for a local teacher kernel and data uniformly sampled in the hypercube $[0,1]^d$. In panels I and J, also the student is a local kernel and the same discussion of \autoref{sec:empirical} applies. In panel K, the student is a convolutional kernel and the generalisation error does not decrease by increasing the size of the training set. Indeed, a local non-shift-invariant function is not on the span of the eigenfunctions of a convolutional kernel, and therefore the student is not able to learn the target.

\begin{figure}
    \centering
    \includegraphics[width=1.0\linewidth]{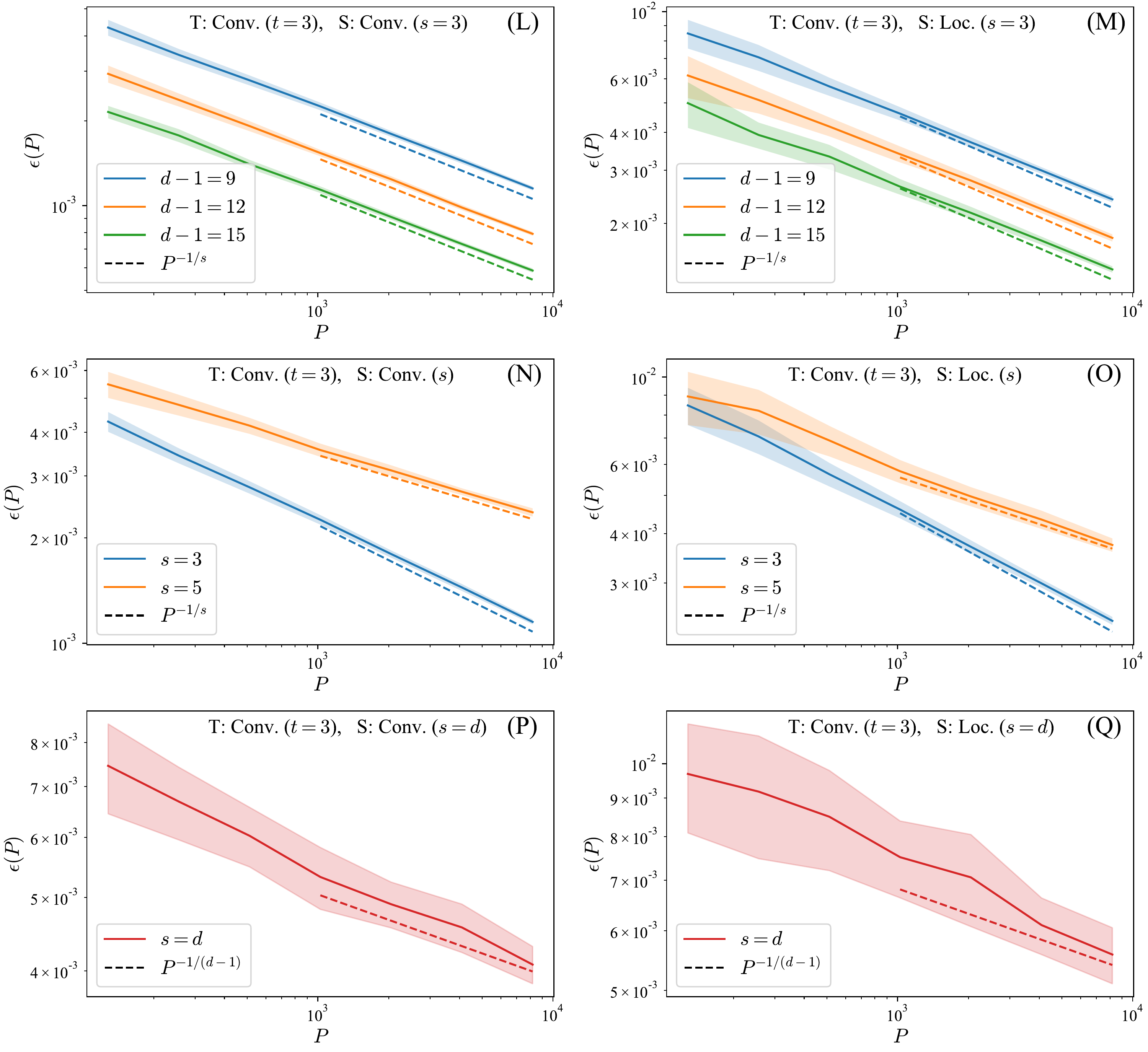}
    \caption{Learning curves for data uniformly distributed on the unit sphere $\mathbb{S}^{d-1}$, with $d=10$ if not specified otherwise. The teacher and student filter sizes are denoted with $t$ and $s$ respectively. Solid lines are the results of numerical experiments averaged over 128 realisations and the shaded areas represent the empirical standard deviations.}
    \label{fig:sphere}
\end{figure}

\paragraph{Spherical data} In \autoref{fig:sphere} we report the empirical learning curves for convolutional teacher and convolutional (left panels) and local (right panels) student kernels. Data are restricted to the unit sphere $\mathbb{S}^{d-1}$. Panels L-O are the analogous of panels A-D of \autoref{fig:figure}. Notice that when the filter size of the student coincides with $d$ (panels P, Q), the learning curves decay with exponent $\beta\,{\eq}\,1/(d-1)$ (instead of $\beta\,{=}\,1/d$). Indeed, for data normalised on $\mathbb{S}^{d-1}$, the spectrum of the Laplacian kernel decays at a rate $\mathcal{O}(k^{-\alpha-(d-1)})$ with $\alpha\,{\eq}\,1$. However, as the student filter size is lowered, we recover the exponent $1/s$ independently of the dimension $d$ of input space, as derived for data on the torus and shown empirically for data in the hypercube. In fact, we expect that the $s$-dimensional marginals of the uniform distribution on $\mathbb{S}^{d-1}$ become insensitive to the spherical constraint when $s\ll d$.

\begin{figure}
    \centering
    \includegraphics[width=1.0\linewidth]{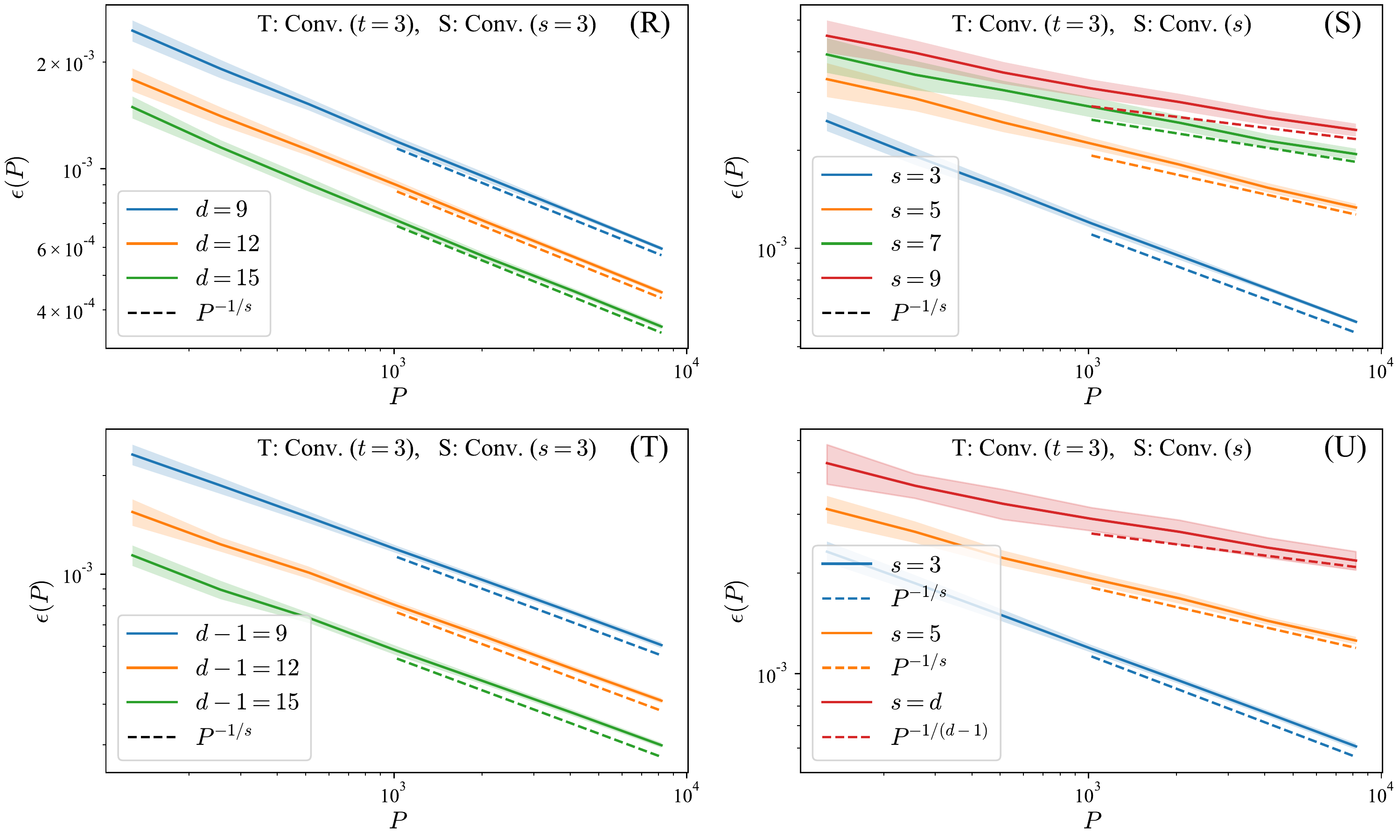}
    \caption{Learning curves for convolutional NTKs and data uniformly distributed in the hypercube $[0,1]^d$ (panels R, S) or on the unit sphere $\mathbb{S}^{d-1}$ (panels T, U). The teacher and student filter sizes are denoted with $t$ and $s$ respectively. Solid lines are the results of numerical experiments averaged over 128 realisations and the shaded areas represent the empirical standard deviations.}
    \label{fig:analytical-ntk}
\end{figure}

\begin{figure}
    \centering
    \includegraphics[width=1.0\linewidth]{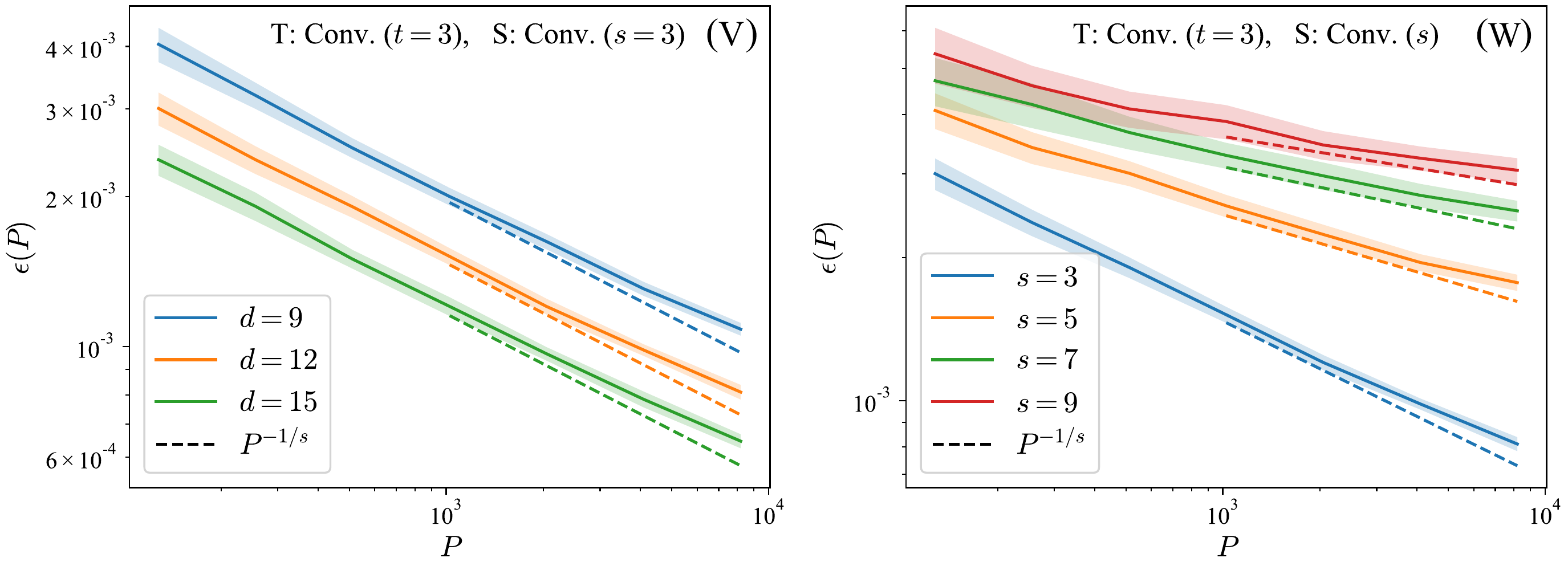}
    \caption{Learning curves for empirical NTKs of very-wide one-hidden-layer CNNs ($H\approx10^6$) and data uniformly distributed in the hypercube $[0,1]^d$. The teacher and student filter sizes are denoted with $t$ and $s$ respectively. Solid lines are the results of numerical experiments averaged over 128 realisations and the shaded areas represent the empirical standard deviations.}
    \label{fig:empirical-ntk}
\end{figure}

\paragraph{Convolutional NTKs} In \autoref{fig:analytical-ntk} we report the empirical learning curves obtained using the NTK of one-hidden-layer CNNs with ReLU activations, which corresponds to using the kernel $\Theta^{FC}$ defined in \autoref{eq:relu-fc-ntk} as the constituent. Since this kernel is not translationally invariant, it cannot be diagonalised in the Fourier domain, and the analysis of \autoref{sec:learning-curves} does not apply. However, as shown in panels P-S, the same learning curve exponents $\beta$ of the Laplacian-constituent case are recovered. Indeed, $\Theta^{FC}$ and the Laplacian kernel share the same nonanalytic behaviour in the origin, and their spectra have the same asymptotic decay \cite{geifman2020similarity}. In \autoref{fig:empirical-ntk} we present the same plots of panels R and S, but instead of the analytical NTKs, we compute numerically the kernels of randomly-initialised very-wide CNNs ($H \approx 10^6$).

\begin{figure}
    \centering
    \includegraphics[width=1.0\linewidth]{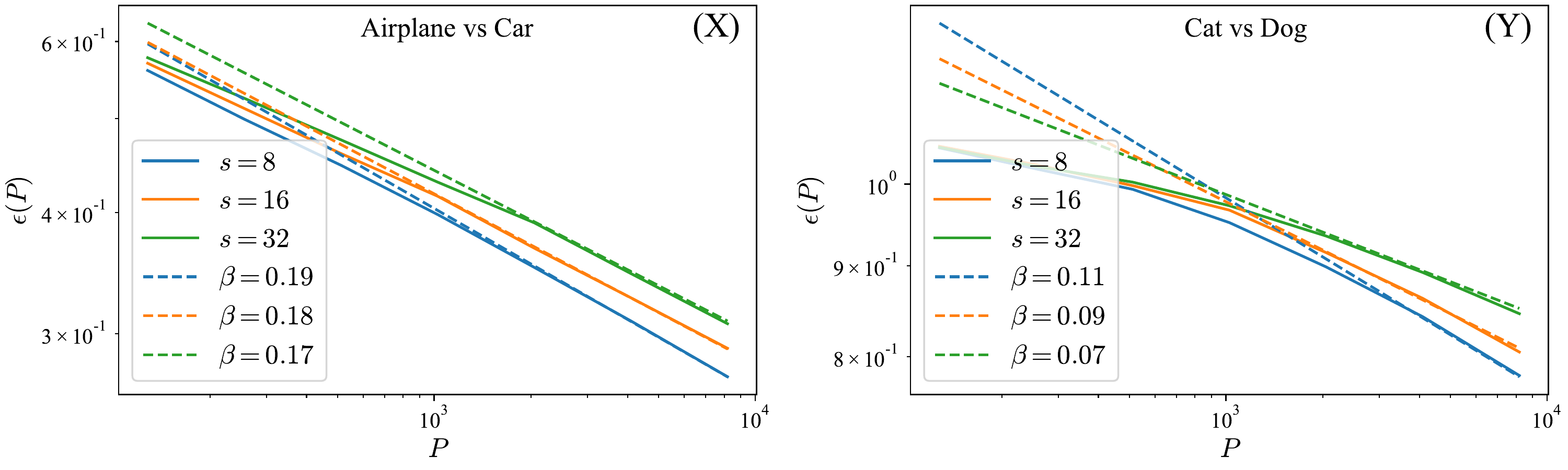}
    \caption{Learning curves of local kernels with filters of size $s$ on CIFAR-10 data. Solid lines are the results of numerical experiments and dashed lines are power laws with exponent $\beta$ interpolated in the last decade.}
    \label{fig:cifar}
\end{figure}

\paragraph{Real data} In Fig. \autoref{fig:cifar} we report the learning curves of local kernels with Laplacian constituents applied to the CIFAR-10 dataset. We build the tasks by selecting two classes and assigning label $+1$ to data from one class and $-1$ to data from the other class. As before, we use $P \in \{128, 256, 512, 1024, 2048, 4096, 8192\}$ and $P_{\text{test}} = 8192$. Differently from our assumptions, image data are strongly anisotropic, and the distance between nearest-neighbour points decays faster than $P^{-1/d}$. Indeed, target functions defined on data of this kind are usually not cursed with the full dimensionality $d$ of the inputs, but rather with an effective dimension $d_{\text{eff}}$. $d_{\text{eff}}$ is related to the dimension of the manifold in which data lie \cite{spigler2019asymptotic}, and may also vary when extracting patches of different sizes. Nonetheless, as we found in our synthetic setup, the learning curve exponent $\beta$ increases monotonically with the filter size of the kernel, strengthening the concept that leveraging locality is key for performance.

\end{document}